\documentclass[accepted]{uai2021} 

\usepackage[american]{babel}

\RequirePackage{amsmath}
\RequirePackage{amssymb}
\RequirePackage{booktabs}
\RequirePackage{color}
\RequirePackage{enumitem} 
\RequirePackage{fancyhdr}
\RequirePackage{graphicx}
\RequirePackage{pgfplots}
\RequirePackage{tikz}
\RequirePackage{url}
\RequirePackage{verbatim}
\RequirePackage{xspace}
\RequirePackage{ifpdf}
\begingroup\expandafter\expandafter\expandafter\endgroup	
\expandafter\ifx\csname IncludeInRelease\endcsname\relax
  \RequirePackage{fixltx2e} 
\fi

\newif\iffinal\finaltrue 
\newif\ifdraft\finalfalse 
\relax

\newcommand{\by}{\@ifstar%
  \BYstar%
  \BYnoStar%
}
\newcommand{\BYstar}[2]{#1\nobreakdash-by\nobreakdash-#2}
\newcommand{\BYnoStar}[2]{\ensuremath{#1\text{\nobreakdash-by\nobreakdash-}#2}}

\ifx \argmax \else
	\DeclareMathOperator*{\argmax}{arg\, max}
\fi

\ifx \argmin \else
	\DeclareMathOperator*{\argmin}{arg\, min}
\fi

\ifx \loglog \else
	\DeclareMathOperator*{\loglog}{log\, log}
\fi

\RequirePackage{mathtools}
\mathtoolsset{showonlyrefs=true}

\providecommand{\note}[1]{{{\color{red} #1}}}
\iffinal
  \ifx\note\undefined
    \newcommand{\note}[1]{{}}%
  \else
    \renewcommand{\note}[1]{{}}%
  \fi
\else
  \ifx\note\undefined
    \newcommand{\note}[2][red]{{{\color{#1} #2}}}%
  \else
    \renewcommand{\note}[2][red]{{{\color{#1} #2}}}%
  \fi
\fi

\definecolor{indigo(web)}{rgb}{0.29, 0.0, 0.51}
\definecolor{internationalorange}{rgb}{1.0, 0.31, 0.0}
\definecolor{green(ryb)}{rgb}{0.4, 0.69, 0.2}
\definecolor{richelectricblue}{rgb}{0.03, 0.57, 0.82}
\definecolor{goldenpoppy}{rgb}{0.99, 0.76, 0.0}
\definecolor{crimson}{rgb}{0.86, 0.08, 0.24}
\definecolor{airforceblue}{rgb}{0.36, 0.54, 0.66}

\definecolor{yafcolor1}{rgb}{0.4, 0.165, 0.553}
\definecolor{yafcolor2}{rgb}{0.949, 0.482, 0.216}
\definecolor{yafcolor3}{rgb}{0.47, 0.549, 0.306}
\definecolor{yafcolor4}{rgb}{0.925, 0.165, 0.224}
\definecolor{yafcolor5}{rgb}{0.141, 0.345, 0.643}
\definecolor{yafcolor6}{rgb}{0.965, 0.933, 0.267}
\definecolor{yafcolor7}{rgb}{0.627, 0.118, 0.165}
\definecolor{yafcolor8}{rgb}{0.878, 0.475, 0.686}

\pgfplotscreateplotcyclelist{yaf}{%
	{yafcolor1},
	{yafcolor2},
	{yafcolor3},
	{yafcolor4},
	{yafcolor5},
	{yafcolor6},
	{yafcolor7},
	{yafcolor8}
}

\pgfplotscreateplotcyclelist{wrongright}{%
	{mambacolor3},
	{mambacolor6}
}

\colorlet{mambacolor1}{indigo(web)}
\colorlet{mambacolor2}{internationalorange}
\colorlet{mambacolor3}{green(ryb)}
\colorlet{mambacolor4}{richelectricblue}
\colorlet{mambacolor5}{goldenpoppy}
\colorlet{mambacolor6}{crimson}
\colorlet{mambacolor7}{airforceblue}

\pgfplotscreateplotcyclelist{mamba}{%
	{mambacolor1},
	{mambacolor2},
	{mambacolor3},
	{mambacolor4},
	{mambacolor5},
	{mambacolor6},
	{mambacolor7}
}

\pgfplotscreateplotcyclelist{mamba-bar}{%
	{indigo(web), fill=indigo(web)},
	{internationalorange, fill=internationalorange},
	{green(ryb), fill=green(ryb)},
	{richelectricblue, fill=richelectricblue},
	{goldenpoppy, fill=goldenpoppy},
	{crimson, fill=crimson},
	{airforceblue, fill=airforceblue}
} 

\pgfplotscreateplotcyclelist{sb-line}{%
	{caribbeangreen},
	{denim},
	{scarlet},
	{green(munsell)},
}

\pgfplotscreateplotcyclelist{sb-bar}{%
	{caribbeangreen, fill=caribbeangreen},
	{denim, fill=denim},
	{scarlet, fill=scarlet},
	{green(munsell), fill=green(munsell)},
}

\pgfplotsset{
	/tikz/normal shift/.code 2 args = {%
		\pgftransformshift{%
			\pgfpointscale{#2}{\pgfplotspointouternormalvectorofticklabelaxis{#1}}%
		}%
	},%
	eda line/.style={
		no markers,
		cycle list name		= mamba,
		tick align        	= outside,
		scaled ticks      	= false,
		enlargelimits     	= false,
		ticklabel shift   	= {10pt},
		axis lines*       	= left,
		line cap          	= round,
		clip              	= false,
		tick style    		= {thin, black, major tick length=2pt},
		x tick label style 	= {font=\scriptsize, yshift = 1pt},
		y tick label style 	= {font=\scriptsize, xshift = 1pt},
		xtick style       	= {normal shift={x}{10pt}},
		ytick style       	= {normal shift={y}{10pt}},
		x axis line style 	= {thick,normal shift={x}{10pt}},
		y axis line style 	= {thick,normal shift={y}{10pt}},
		x label style 		= {at={(axis description cs:0.5,0)}, normal shift={x}{16pt}, anchor=north, font=\scriptsize},
		y label style 		= {at={(axis description cs:0,0.5)}, normal shift={y}{24pt}, anchor=south, font=\scriptsize},
		legend cell align 	= left,
		legend style 		= {inner sep = 1pt, cells = {font=\scriptsize}, },
		legend image code/.code={%
			\draw[mark repeat=2,mark phase=2,#1] 
			plot coordinates { (0cm,0cm) (0.15cm,0cm) (0.3cm,0cm) };%
		}
	}
}
	
\pgfplotsset{
	eda ybar/.style={
		ybar,
		area legend,
		ymajorgrids,
		no markers,
		axis on top,
		xtick				= data,
		cycle list name    	= mamba-bar,
		tick align        	= outside,
		enlargelimits     	= false,
		xmajorgrids 		= false,
		bar width			= 0.7em,
		major grid style	= white,
		axis lines* 		= left,			
		tick style    		= {thin, black, major tick length=2pt},
		major y tick style	= {draw=none},
		x tick label style 	= {font=\scriptsize, yshift=1pt},
		y tick label style = {font=\scriptsize, xshift=1pt},
		x axis line style 	= {thick, normal shift={x}{0pt}},
		y axis line style	= {opacity=0},	
		x label style 		= {at={(axis description cs:0.5,0)}, normal shift={x}{6pt}, anchor=north, font=\scriptsize},
		y label style 		= {at={(axis description cs:0,0.5)}, normal shift={y}{14pt}, anchor=south, font=\scriptsize},
		legend image post style={scale=0.25},
		legend style 		= {inner sep=1pt, cells={font=\scriptsize}, },
		legend cell align 	= left
	}
}

\pgfplotsset{
	eda ybar log/.style={
		eda ybar,
		area legend,
		ymajorgrids,
		no markers,
		axis on top,
		xtick				= data,
		cycle list name    	= mamba-bar,
		tick align        	= outside,
		enlargelimits     	= false,
		xmajorgrids 		= false,
		bar width			= 0.7em,
		major grid style	= white,
		axis lines* 		= left,			
		tick style    		= {thin, black, major tick length=2pt},
		minor y tick style	= {draw=none},
		major y tick style	= {draw=none},
		x tick label style 	= {font=\scriptsize, yshift=1pt},
		y tick label style = {font=\scriptsize, xshift=1pt},
		x axis line style 	= {thick, normal shift={x}{0pt}},
		y axis line style	= {opacity=0},	
		x label style 		= {at={(axis description cs:0.5,0)}, normal shift={x}{10pt}, anchor=north, font=\scriptsize},
		y label style 		= {at={(axis description cs:0,0.5)}, normal shift={y}{4pt}, anchor=south, font=\scriptsize},
		legend image post style={scale=0.25},
		legend style 		= {inner sep=1pt, cells={font=\scriptsize}, },
		legend cell align 	= left
	}
}

\pgfplotsset {
	eda scatter3/.style={
		only marks,
		mark				= *,
		cycle list name		= mamba,
		scaled ticks      	= false,
		enlargelimits     	= false,
		axis lines*			= left,
		mark size			= 0.75pt,
		tick pos 			= left,
		tick align			= outside,
		ticklabel shift   	= {10pt},
		clip              	= false,
		tick style    		= {thin, black, major tick length=2pt},
		xtick style       	= {normal shift={x}{10pt}},
		ytick style       	= {normal shift={y}{10pt}},
		x tick label style 	= {font=\scriptsize, yshift = 1pt},
		y tick label style 	= {font=\scriptsize, xshift = 1pt},
		x axis line style 	= {thick, normal shift={x}{10pt}},
		y axis line style 	= {thick,normal shift={y}{10pt}},
		x label style 		= {at={(axis description cs:0.5,0)}, normal shift={x}{0pt}, anchor=north, font=\scriptsize},
		y label style 		= {at={(axis description cs:0,0.5)}, normal shift={y}{-24pt}, anchor=south, font=\scriptsize},
		scatter/use mapped color={draw=indigo(web),fill=indigo(web)},
		legend cell align 	= left,
		legend style 		= {inner sep = 1pt, cells = {font=\scriptsize}, },
		legend image code/.code={%
			\draw[mark repeat=2,mark phase=2,#1] 
			plot coordinates { (0cm,0cm) (0.15cm,0cm) (0.3cm,0cm) };%
		}
	}
}

\pgfplotsset {
	eda scatter4/.style={
		mark				= *,
		cycle list name		= mamba,
		scaled ticks      	= false,
		enlargelimits     	= false,
		axis lines*			= left,
		mark size			= 0.75pt,
		tick pos 			= left,
		tick align			= outside,
		ticklabel shift   	= {10pt},
		clip              	= false,
		tick style    		= {thin, black, major tick length=2pt},
		xtick style       	= {normal shift={x}{10pt}},
		ytick style       	= {normal shift={y}{10pt}},
		x tick label style 	= {font=\scriptsize, yshift = 1pt},
		y tick label style 	= {font=\scriptsize, xshift = 1pt},
		x axis line style 	= {thick, normal shift={x}{10pt}},
		y axis line style 	= {thick,normal shift={y}{10pt}},
		x label style 		= {at={(axis description cs:0.5,0)}, normal shift={x}{12pt}, anchor=north, font=\scriptsize},
		y label style 		= {at={(axis description cs:0,0.5)}, normal shift={y}{16pt}, anchor=south, font=\scriptsize},
		scatter/use mapped color={draw=indigo(web),fill=indigo(web)},
		legend cell align 	= left,
		legend style 		= {inner sep = 1pt, cells = {font=\scriptsize}, },
		legend image code/.code={%
			\draw[mark repeat=2,mark phase=2,#1] 
			plot coordinates { (0cm,0cm) (0.15cm,0cm) (0.3cm,0cm) };%
		}
	}
}

\pgfplotsset{
    box plot/.style={
        /pgfplots/.cd,
        black,
        only marks,
        mark=-,
        tick style    		= {thin, black, major tick length=2pt},
	major y tick style	= {draw=none},
	x tick label style 	= {font=\scriptsize, yshift=1pt},
	y tick label style = {font=\scriptsize, xshift=-2pt},
	x axis line style 	= {thick, normal shift={x}{0pt}},
	y axis line style	= {opacity=0},	
	x label style 		= {at={(axis description cs:0.5,0)}, normal shift={x}{6pt}, anchor=north, font=\scriptsize},
	y label style 		= {at={(axis description cs:0,0.5)}, normal shift={y}{20pt}, anchor=south, font=\scriptsize},
	xtick pos = left,
	ytick pos = left,
        mark size=\pgfkeysvalueof{/pgfplots/box plot width},
        /pgfplots/error bars/y dir=plus,
        /pgfplots/error bars/y explicit,
        /pgfplots/table/x index=\pgfkeysvalueof{/pgfplots/box plot x index},
    },
    box plot box/.style={
        /pgfplots/error bars/draw error bar/.code 2 args={%
            \draw  ##1 -- ++(\pgfkeysvalueof{/pgfplots/box plot width},0pt) |- ##2 -- ++(-\pgfkeysvalueof{/pgfplots/box plot width},0pt) |- ##1 -- cycle;
        },
        /pgfplots/table/.cd,
        y index=\pgfkeysvalueof{/pgfplots/box plot box top index},
        y error expr={
            \thisrowno{\pgfkeysvalueof{/pgfplots/box plot box bottom index}}
            - \thisrowno{\pgfkeysvalueof{/pgfplots/box plot box top index}}
        },
        /pgfplots/box plot
    },
    box plot top whisker/.style={
        /pgfplots/error bars/draw error bar/.code 2 args={%
            \pgfkeysgetvalue{/pgfplots/error bars/error mark}%
            {\pgfplotserrorbarsmark}%
            \pgfkeysgetvalue{/pgfplots/error bars/error mark options}%
            {\pgfplotserrorbarsmarkopts}%
            \path ##1 -- ##2;
        },
        /pgfplots/table/.cd,
        y index=\pgfkeysvalueof{/pgfplots/box plot whisker top index},
        y error expr={
            \thisrowno{\pgfkeysvalueof{/pgfplots/box plot box top index}}
            - \thisrowno{\pgfkeysvalueof{/pgfplots/box plot whisker top index}}
        },
        /pgfplots/box plot
    },
    box plot bottom whisker/.style={
        /pgfplots/error bars/draw error bar/.code 2 args={%
            \pgfkeysgetvalue{/pgfplots/error bars/error mark}%
            {\pgfplotserrorbarsmark}%
            \pgfkeysgetvalue{/pgfplots/error bars/error mark options}%
            {\pgfplotserrorbarsmarkopts}%
            \path ##1 -- ##2;
        },
        /pgfplots/table/.cd,
        y index=\pgfkeysvalueof{/pgfplots/box plot whisker bottom index},
        y error expr={
            \thisrowno{\pgfkeysvalueof{/pgfplots/box plot box bottom index}}
            - \thisrowno{\pgfkeysvalueof{/pgfplots/box plot whisker bottom index}}
        },
        /pgfplots/box plot
    },
    box plot median/.style={
        /pgfplots/box plot,
        /pgfplots/table/y index=\pgfkeysvalueof{/pgfplots/box plot median index}
    },
    box plot width/.initial=1em,
    box plot x index/.initial=0,
    box plot median index/.initial=1,
    box plot box top index/.initial=2,
    box plot box bottom index/.initial=3,
    box plot whisker top index/.initial=4,
    box plot whisker bottom index/.initial=5,
}

\pgfplotsset{
	eda surf/.style={
		view={56}{26},
		axis lines=left,		
		tick style    		= {thin, black, major tick length=2pt},
		 xmajorgrids, x dir= reverse, ymajorgrids, zmajorgrids,
		minor y tick style	= {draw=none},
		major y tick style	= {draw=none},
		major z tick style	= {draw=none},
		x tick label style 	= {font=\scriptsize, yshift=1pt},
		y tick label style 	= {font=\scriptsize, xshift=-3pt, yshift=3pt},
		z tick label style 	= {font=\scriptsize, xshift=1pt},
		z axis line style 		= {thick, normal shift={x}{0pt}},
		x axis line style		= {opacity=0},	
		y axis line style		= {opacity=0},	
		z label style 		= {at={(axis description cs:-0.15,0.5)}, normal shift={z}{10pt}, anchor=south, font=\scriptsize},
		x label style 		= {at={(axis description cs:-0.15,-0.15)}, anchor=south, rotate=-35, font=\scriptsize},
		y label style 		= {at={(axis description cs:0.85,-0.2)}, rotate=15, anchor=south, font=\scriptsize},
		legend image post style={scale=0.25},
		legend style 		= {inner sep=1pt, cells={font=\scriptsize}, },
		legend cell align 	= left,
		grid=major,
		colormap={reverse hot}{
        			indices of colormap={
	            		\pgfplotscolormaplastindexof{hot},...,0 of hot}
    		}
	}
}
\usepackage[utf8]{inputenc}
\usepackage[T1]{fontenc}
\usepackage[ruled, vlined, nofillcomment, linesnumbered]{algorithm2e}
\usepackage{cancel}
\usepackage{txfonts}  
\usepackage{mathtools} 
\usepgflibrary{arrows}
\usepackage{subcaption}
\usepackage{bm} 
\usepackage{color}
\usepackage{arydshln}	
\usepackage{newclude}
\usepackage{pdfpages} 

\usepackage{thmtools}  

\usepackage{times}

\usepackage{hyperref}
\hypersetup{colorlinks,%
citecolor=black,%
filecolor=black,%
linkcolor=black,%
urlcolor=black
}
\usepackage[round]{natbib}
\newif\ifapx
\apxtrue 

\newif\iflong
\longfalse 

\setlist[enumerate]{noitemsep, nolistsep}
\setlist[itemize]{noitemsep, nolistsep}

\SetKwComment{tcpas}{\{}{\}}
\SetCommentSty{textnormal}
\SetArgSty{textnormal}
\SetKw{False}{false}
\SetKw{True}{true}
\SetKw{Null}{null}
\SetKwInOut{Output}{output}
\SetKwInOut{Input}{input}
\SetKw{AND}{and}
\SetKw{OR}{or}
\SetKw{Continue}{continue}

\makeatletter
\newcommand{\oset}[3][0ex]{%
  \mathrel{\mathop{#3}\limits^{
    \vbox to#1{\kern-2\ex@
    \hbox{$\scriptstyle#2$}\vss}}}}
\makeatother
\newcommand{\aone}{\ensuremath{-_1}\xspace}
\newcommand{\atwo}{\ensuremath{-_2}\xspace}
\newcommand{\aletwo}{\ensuremath{-_{\le 2}}\xspace}
\newcommand{\atwos}{\ensuremath{\oset{s}{-}_2}\xspace}
\newcommand{\aletwos}{\ensuremath{\oset{s}{-}_{\le 2}}\xspace}

\DeclarePairedDelimiterX{\infdivx}[2]{(}{)}{%
  #1\;\delimsize\|\;#2%
}

\DeclareMathAlphabet{\mathbcal}{OMS}{cmsy}{b}{n} 

\DeclareMathOperator*{\argmin}{\arg\!\min}
\DeclareMathOperator*{\argmax}{\arg\!\max}

\newcommand{\Models}{\ensuremath{\mathcal{M}}\xspace}

\newcommand{\Indep}{\mathop{\perp\!\!\!\perp}\nolimits} 
\newcommand{\nIndep}{\mathop{\cancel\Indep}\nolimits}

\newcommand{\Pa}{\ensuremath{\text{Pa}}\xspace}
\newcommand{\An}{\ensuremath{\text{An}}\xspace}
\newcommand{\Sp}{\ensuremath{\text{Sp}}\xspace}

\newcommand{\Ch}{\ensuremath{\text{Ch}}\xspace}
\newcommand{\De}{\ensuremath{\text{De}}\xspace}
\newcommand{\Nd}{\ensuremath{\text{Nd}}\xspace}
\newcommand{\PC}{\ensuremath{\text{PC}}\xspace}
\newcommand{\MB}{\ensuremath{\text{MB}}\xspace}

\newcommand{\variables}{\ensuremath{\bm{V}}\xspace}

\newtheorem{corollary}{Corollary}

\newtheorem{assumption}{Assumption}
\newtheorem{definition}{Definition}
\newtheorem{example}{Example}
\newenvironment{proof}{\paragraph{Proof:}}{\hfill$\square$}

\tikzset{node/.style={black, draw=black, circle, minimum size=0.7cm, scale=0.8}} 
\tikzset{dummy/.style={black, draw=black, circle, minimum size=0.55cm, scale=0.7}} 
\tikzset{latent/.style={black, draw=black, fill=lightgray, circle, minimum size=0.55cm, scale=0.7}} 

\tikzset{causes/.style={->,very thick,  color=black}} 
\tikzset{causesxor/.style={->,very thick, dashed, color=black}} 
\tikzset{causesxoro/.style={o->,very thick, dashed, color=black}} 
\tikzset{connected/.style={o-o,very thick, color=black}} 
\tikzset{connectedd/.style={o-o,very thick, dashed,  color=black}} 
\tikzset{ocauses/.style={o->,very thick,  color=black}} 
\tikzset{confounder/.style={<->,very thick,  color=black}} 
\tikzset{confounderxor/.style={<->,very thick, dashed, color=black}} 
\tikzset{confounderl/.style={<->,very thick,  color=black, bend left=45}} 
\tikzset{confounderr/.style={<->,very thick,  color=black, bend right=45}}

\graphicspath{ {./figs/} }

\ifpdf
\usetikzlibrary{external}
\fi

\title{A Weaker Faithfulness Assumption based on Triple Interactions}

\author[1,2]{\href{mailto:<alexander.marx@cispa.de>?Subject=Your UAI 2021 paper}{Alexander~Marx}{}}
\author[3]{Arthur~Gretton}
\author[4]{Joris M. Mooij}
\affil[1]{%
    CISPA Helmholtz Center for Information Security\\
    Saarland University\\
    Saarbr{\"u}cken, Germany
}
\affil[2]{%
    Max Planck Institute for Informatics\\
    Saarland University\\
    Saarbr{\"u}cken, Germany
}
\affil[3]{%
    Gatsby Unit\\
    University College London\\
    London, United Kingdom
  }
\affil[4]{%
  Korteweg-de Vries Institute\\
  University of Amsterdam\\
  Amsterdam, The Netherlands
}

\begin{document}

\maketitle

\begin{abstract}
One of the core assumptions in causal discovery is the faithfulness assumption---i.e.~assuming that independencies found in the data are due to separations in the true causal graph. This assumption can, however, be violated in many ways, including xor connections, deterministic functions or cancelling paths. In this work, we propose a weaker assumption that we call $2$-adjacency faithfulness. In contrast to adjacency faithfulness, which assumes that there is no conditional independence between each pair of variables that are connected in the causal graph, we only require no conditional independence between a node and a subset of its Markov blanket that can contain up to two nodes. Equivalently, we adapt orientation faithfulness to this setting.
We further propose a sound orientation rule for causal discovery that applies under weaker assumptions. As a proof of concept, we derive a modified Grow and Shrink algorithm that recovers the Markov blanket of a target node and prove its correctness under strictly weaker assumptions than the standard faithfulness assumption.
\end{abstract}

\section{INTRODUCTION}
\label{sec:introduction}

In this work, we focus on causal discovery from observational data, where we are given a sample from the joint distribution $P$ of the observed variables and try to infer the true causal graph $G$ between them. Two standard assumptions in this field are the causal Markov condition and the faithfulness assumption~\citep{spirtes:00:book}. While the causal Markov condition assumes that all separations in the true causal graph $G$ imply independencies in $P$, the faithfulness assumption is its counterpart.
That is, all independencies found in $P$ are due to separations in $G$. Although both assumptions have great merit for causal discovery algorithms, especially the faithfulness assumption has been criticized in the past~\citep{andersen:13;violationmatters,zhang:16:threefacesofff}.

\begin{figure}[t]%
	\begin{minipage}[t]{.5\linewidth}
		\centering
		\includegraphics[]{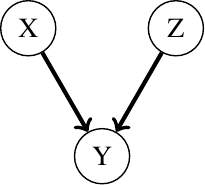}
		\subcaption{}
	\end{minipage}%
	\begin{minipage}[t]{.5\linewidth}
		\centering
		\includegraphics[]{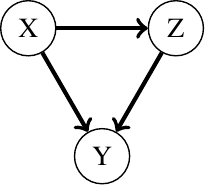}
		\subcaption{}
	\end{minipage}%
	\caption{Failures of adjacency faithfulness: Assume in graph (a) $X,Z$ are fair independent coins and $Y := (X \oplus Z) \oplus E$, where $\oplus$ is the xor operator and $E$ is a biased coin denoting a noise term. Then $X$ is independent of $Y$ (denoted as $X \Indep_P Y$) and $Z \Indep_P Y$. Graph (b) could correspond with a linear model where both directed paths from $X$ to $Y$ cancel 
	\iflong
\else
	out such that $X \Indep_P Y$, 
	\fi
	but $X \nIndep_P Z$ and $Z \nIndep_P Y$.}
	\label{fig:adjacency-failure}
\end{figure}

Despite it was proven that faithfulness violations in causally sufficient linear-Gaussian and discrete acyclic systems occur with Lebesgue measure zero~\citep{meek:95:zero}, it has also been shown that on a finite sample, empirical faithfulness violations do appear surprisingly often~\citep{uhler:13:geometry}. Even on population level, there exist simple generating mechanisms, as shown in Figure~\ref{fig:adjacency-failure}, that violate the faithfulness assumption. For instance, two independent random variables $X$ and $Z$, that can be modelled by fair coins, together cause $Y$ through a noisy xor relation. As a consequence, all three variables are marginally independent. Following the faithfulness assumption, there should be no edges connecting $X,Y$ and $Z$ in the causal graph---however, there are. 

Faithfulness violations like the above have been intensively studied in the past~\citep{ramsey:06:cpc,zhang:08:triangle,spirtes:14:ktriangle} and several weaker assumptions such as adjacency faithfulness~\citep{spirtes:00:book}, P-minimality~\citep{pearl:09:causality}, SGS-minimality~\citep{spirtes:00:book} and frugality~\citep{forster:17:frugality}, which we review in Section~\ref{sec:related}, have been proposed. Although faithfulness violations induced by xor-type relations---i.e. both parents are marginally independent of the child node---can be detected by most of the above approaches, they do not analyze under which conditions the DAG structure can be partially recovered, once such violations have been detected.

In this work, we propose a new assumption that we call \textit{$2$-adjacency faithfulness}, which allows us to both detect such faithfulness violations and partially infer the underlying DAG structure under certain conditions. We start by explaining the standard concepts and notation in Section~\ref{sec:preliminaries} and review failures of faithfulness as well as related work in Section~\ref{sec:problem}. Then, we study the causal structure of xor-type connections in Section~\ref{sec:ufts} and propose $2$-adjacency faithfulness in Section~\ref{sec:connections}. To partially infer causal DAGs that may contain such generating mechanisms, we introduce a sound orientation rule, in Section~\ref{sec:orientation}. Further, we show under which assumptions on the distribution this rule is applicable---which we formalize as \textit{$2$-orientation faithfulness}---and analyze its failure cases. As a proof of concept, we 
\iflong
\else
provide a modified Grow and Shrink (GS)
\fi
algorithm~\citep{margaritis:00:gs} in Section~\ref{sec:implementation} and show it correctly identifies the Markov blanket of a target node under strictly weaker assumptions than faithfulness. Besides, we give some intuition on how to extend well-known causal discovery algorithms based on our new assumptions.

\section{DAGS AND INDEPENDENCE}
\label{sec:preliminaries}

In this section, we define the notation and provide definitions for separations on graphs and independence.

\subsection{Causal Graphs}

A causal \textit{directed acyclic graph} (DAG) $G$ over a set of random variables $\variables$ with joint distribution $P$ is defined such that each pair of nodes that is adjacent in $G$ is causally related. For simplicity, we will use the random variables $\variables$ to also refer to the nodes of the graph. A directed edge $X \rightarrow Y$ in $G$ between two nodes representing the random variables $X,Y \in \variables$ indicates that $X$ is a \textit{direct cause} or \textit{parent} of $Y$ and that $Y$ is a \textit{direct effect} or \textit{child} of $X$. Accordingly, we denote the set of all parents of $X \in \variables$ with $\Pa(X)$, the set of all children with $\Ch(X)$ and the set of parents and children with $\PC(X) := \Pa(X) \cup \Ch(X)$. Further, we write $\An(X)$ for the set of \textit{ancestors} and denote its  \textit{descendants} by $\De(X)$, where $X$ is an ancestor and descendant of itself. Respectively, we refer to the \textit{non-descendants} of $X$ as $\Nd(X) := \variables \backslash \De(X)$. Last, the \textit{Markov blanket} of $X$ is defined as  $\MB(X) := \PC(X) \cup \Sp(X)$, where $\Sp(X)$ are the spouses of $X$, that is, nodes that share a child node with $X$. Importantly, $X$ is $d$-separated of any other node in the graph given its Markov blanket and $\MB(X)$ is the smallest such set.

DAGs are used to represent causal graphs under the assumption of acyclicity, no selection bias, and \textit{causal sufficiency}, that is, it is assumed that no two variables $X, Y \in \variables$ are caused by an unobserved confounder $Z \not \in \variables$. This is also the setup on which we focus in this paper---i.e. assuming that all relevant variables are observed, that there are no causal cycles and that there has been no conditioning on selection variables.
Further, as a short form to summarize a model as defined above, we write $\Models = (G, \variables, P)$.

\subsection{Independence and Separation}

In the following, we define conditional independence in a probability distribution and $d$-separation in a graph. 

Given three sets of random variables $\bm{X}, \bm{Y}, \bm{Z} \subseteq \variables$, where $P$ is the joint distribution over $\variables$, we denote that $\bm{X}$ is \textit{probabilistically independent} of $\bm{Y}$ given $\bm{Z}$ in $P$ as $\bm{X} \Indep_P \bm{Y} \mid \bm{Z}$. 

\textit{$D$-separation}~\citep{pearl:09:causality} is defined in terms of paths. 
A \textit{path} $p$ between $X$ and $Y$, denoted $p = \langle X, \dots, Y \rangle$, is a sequence of distinct nodes $X_1, \dots, X_n$ such that $X_i$ is adjacent to $X_{i+1}$ for $i=1, \dots, n-1$, $X_1 = X$ and $X_n =Y$.
Further, we call a node $C$ a \textit{collider} on a path $\langle \dots, X, C, Y, \dots \rangle$, where $C$ is adjacent to both $X$ and $Y$, if two arrowheads point to it, that is $X \rightarrow C \leftarrow Y$.

\begin{definition}[$d$-Separation] 
A path between two vertices $X,Y$ in a DAG is $d$-connecting given a set $\bm{Z}$, if
\begin{enumerate}
	\item every non-collider on the path is not in  $\bm{Z}$, and
	\item every collider on the path is an ancestor of $\bm{Z}$.
\end{enumerate}
If there is no path $d$-connecting $X$ and $Y$ given $\bm{Z}$, then $X$ and $Y$ are $d$-separated given $\bm{Z}$. Sets $\bm{X}$ and $\bm{Y}$ are $d$-separated given $\bm{Z}$, if for every pair $X, Y$, with $X \in \bm{X}$ and $Y \in \bm{Y}$, $X$ and $Y$ are $d$-separated given $\bm{Z}$.
\end{definition}

As shorthand notation for separations on a DAG $G$, we write $\bm{X} \Indep_G \bm{Y} \mid \bm{Z}$ if $\bm{X}$ is $d$-separated from $\bm{Y}$ given $\bm{Z}$.
\iflong
\else
A useful set of tools for inferences on graphs and distributions are the graphoid axioms~\citep{dawid:79:graphoid,spohn:80:stochastic,geiger:90:identifying}. Since we use those axioms in our proofs, we provide them in Supplementary Material~\ref{sup:proofs}.
\fi

We round up this section by defining the causal Markov condition~\citep{spirtes:00:book} (CMC) for DAGs.

\begin{definition}[Causal Markov Condition]
Given the tri\-ple $\Models = (G, \variables, P)$, the causal Markov condition holds, if every $d$-separation in $G$ implies an independence in $P$.
\end{definition}

The causal Markov condition is one of the most essential assumptions for causal discovery algorithms. 
On the other hand, assumptions about what properties of the graph can be inferred based on the given distribution have been weakened over time~\citep{ramsey:06:cpc,zhang:08:triangle,forster:17:frugality}. Most commonly known is the faithfulness assumption.

\section{ADJACENCY FAITHFULNESS AND WHEN IT IS VIOLATED}
\label{sec:problem}

To lay out the problem, we first explain faithfulness and adjacency faithfulness, then examine when those could fail and give a summary about the most relevant related approaches that use weaker assumptions.

The faithfulness assumption is one of the core assumptions made by most causal discovery algorithms~\citep{spirtes:00:book} and it can be seen as the inverse assumption to CMC---i.e. assuming that all independencies found in $P$ imply a $d$-separation in the causal graph. Adjacency faithfulness is a slightly weaker assumption.

\begin{definition}[Adjacency Faithfulness]
\iflong
\else
Let $\Models = (G, \variables, P)$,
\fi
if $X, Y \in \variables$ are adjacent in $G$, then they are probabilistically dependent given all $\bm{S} \subseteq \variables \backslash \{ X, Y \}$.
\end{definition}

Alternatively, we could turn this definition around by stating that if we find a conditional independence in $P$, then we assume that there is no edge in the corresponding graph. Assuming adjacency faithfulness ensures that we recover the correct skeleton graph (i.e.~the undirected graph). Correct detection of the skeleton together with the correct identification of all collider structures ensures that the detected graph is in the Markov equivalence class of the true graph~\citep{verma:91:equivalence}. 
\iflong
\else
The latter is ensured by additionally assuming orientation faithfulness~\citep{zhang:08:triangle}.
\fi

\begin{definition}[Orientation-Faithfulness]
\iflong
\else
Let $\Models = (G, \variables, P)$ and let $\langle X,Y,Z \rangle$ be un unshielded path\footnote{For an unshielded path $\langle X,Y,Z \rangle$, $X$ is adjacent to $Y$ and $Y$ is adjacent to $Z$, but $X$ is not adjacent to $Z$.} in $G$.
\fi
\begin{enumerate}
	\item If $X \rightarrow Y \leftarrow Z$, then $X$ and $Z$ are dependent given any subset in $\variables \backslash \{ X, Z \}$ that contains $Y$; otherwise
	\item $X$ and $Z$ are dependent conditional on any subset of $\variables \backslash \{ X, Z \}$ that does not contain $Y$.
\end{enumerate}
\end{definition}

The bottleneck here is the adjacency faithfulness assumption, as many causal discovery algorithms such as PC~\citep{spirtes:00:book} or GES~\citep{chickering:02:ges} rely on finding adjacent nodes either by checking for marginal dependencies or adding single edges based on adjacency faithfulness and CMC.
However, adjacency faithfulness can be violated in many ways, e.g.~by xor-type connections, path cancellations, or deterministic relations. We briefly explain the first two below, as they are relevant for the remainder. For deterministic relations and finite sample failures, we refer to \cite{lemeire:12:conservative}.

\subsection{Xor-Type Relations}

\iflong
\else
\begin{figure}
	\centering
	\begin{minipage}[t]{.5\linewidth}
	\centering
	\includegraphics[]{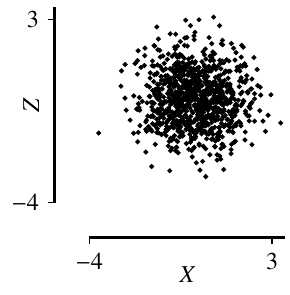}
   	\end{minipage}%
       	\begin{minipage}[t]{.5\linewidth}
	\centering
	\includegraphics[]{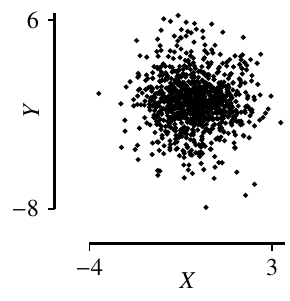}
   	\end{minipage}%
	\newline
	\begin{minipage}[t]{.5\linewidth}
	\centering
	\includegraphics[]{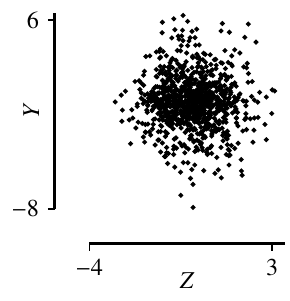}
   	\end{minipage}%
	\begin{minipage}[t]{.5\linewidth}
	\centering
	\includegraphics[]{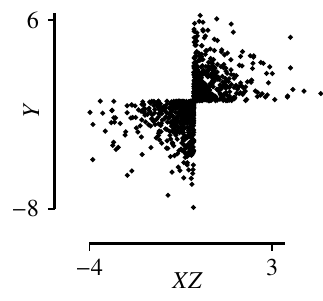}
   	\end{minipage}%
	\caption{Sample data for the collider graph $X \to Y \leftarrow Z$, where $X,Z \sim N(0,1)$ are iid and $Y := \text{sign}(XZ) \cdot E$, with $E \sim \text{Exp}(\frac{1}{\sqrt{2}})$. The dependence is only detectable by considering all three variables jointly.}
	\label{fig:explosion-example}
\end{figure}
\fi

In this work, we focus on xor-type relations. That is, given a triple of nodes $X,Y,Z \in \variables$ s.t. $X \rightarrow Y \leftarrow Z$, where at least one of the causal edges cannot be detected by a marginal dependence, but only by looking at the joint distribution of $X,Y$ and $Z$.
The key here is that either parent of $Y$ might not be dependent on $Y$, but by considering both parents, we can detect the dependence. To illustrate this, consider the following example where we describe a noisy xor with an unobserved noise variable modelled by a biased coin as it is common for binary causal structures~\citep{inazumi:11:exclusiveor}.

\begin{example}
\label{ex:one}
\iflong
\else
Let $\Models = (G, \variables, P)$ be a causal model. Given variables $X,Y,Z \in \variables$ s.t.~$X \rightarrow Y \leftarrow Z$
\fi
in $G$ and there is no edge connecting $X$ and $Z$, as in Figure~\ref{fig:adjacency-failure}(a), where $X, Z$ are fair independent coins. Their common effect $Y$ is generated as $Y := (X \oplus Z) \oplus E$, where $\oplus$ denotes xor---i.e. $X \oplus Z := (X + Z) \mod 2$---and $E$ is a biased coin with $P(E = 1) = p$, where $0 \le p < \frac{1}{2}$ and $E \Indep_P \{ X, Z \}$. Hence, $X \nIndep_G Y$, $Z \nIndep_G Y$, however, due to the xor, we have that $X \Indep_P Y$ and $Z \Indep_P Y$. Both edges violate adjacency faithfulness. If we were to check the joint distribution, we can find that $Y \nIndep_P \{ X, Z \}$, or $X \nIndep_P Z \mid Y$, since we get that $P(X=1,Z=1,Y=1) = \frac{p}{4}$, where $P(X=1,Z=1) \cdot P(Y=1) = \frac{1}{4} \cdot \frac{1}{2} = \frac{1}{8}$. Those terms are only equal if $p = \frac{1}{2}$, which we excluded by assumption.
\end{example}

Similar examples, where marginal dependencies might be hard to detect, can also be constructed for continuous data~\citep{sejdinovic:13:lancasterkernel}---e.g.~if $X,Z$ are normally distributed with mean zero and variance one, and $Y := \text{sign}(XZ) \cdot E$, with exponentially distributed noise $E \sim \text{Exp}(\frac{1}{\sqrt{2}})$ (see Figure~\ref{fig:explosion-example}). The authors demonstrate that with a high probability no marginal dependence between $X,Z$ as well as $X,Y$ can be detected using a kernel dependence measure such as the Hilbert-Schmidt Independence Criterium~\citep{gretton:05:hsic}. They do, however, reliably detect a dependence between all three variables. Further, \cite{marx:21:myl} show that dependencies generated by the above mechanism and generalizations of it to mixed-type data can be detected by a broad range of dependence measures such as kernel tests and tests based on conditional mutual information. Besides iid data, the work of \cite{sejdinovic:13:lancasterkernel} has been extended to time series data~\citep{rubenstein:16:lancaster-ts}. In this setting, \cite{rubenstein:16:lancaster-ts} showcase the efficiency of their approach on a Forex data set, which contains strong triple interactions.

Motivated by these positive results, we focus on the theoretical foundations that allow us to detect such triples in a causal setting and anlaylze under which conditions we can identify the collider.

\subsection{Cancelling Paths}
\label{sec:cancelling:path}

A minimal example of cancelling paths was given by \cite{hesslow:76:examples} and is illustrated with the causal graph shown in Figure~\ref{fig:adjacency-failure}(b). In Hasslow's example taking birth control pills ($X$) can influence the risk of getting thrombosis ($Y$) via two paths. It has a direct effect and also taking the pills reduces the chance of pregnancy ($Z$), which itself is a cause of thrombosis. However, the causal effects induced by those paths cancel
\iflong
\else
such that $X \Indep_P Y$.
\fi
As an example mechanism that causes such a cancellation, consider a linear Gaussian system in which $Z := \alpha X$, $Y := \beta Z - \gamma X$ and $\gamma = \alpha \beta$.
This failure of faithfulness was shown to be undetectable since $X$ will be dependent on $Y$ given $Z$ and hence the graph $X \rightarrow Z \leftarrow Y$ is also a valid graph for those independencies---i.e.~Markov equivalent~\citep{zhang:08:triangle}. There exist cancelling paths that consist of more than three variables, which are detectable, e.g. if $Z$ is not adjacent to $Y$, but there is a path $Z \to W \to Y$~\citep{zhang:08:triangle}. 

\subsection{Weaker Assumptions}
\label{sec:related}

In the following, we discuss different approaches on how to relax the faithfulness assumption.

Two well-studied assumptions are P-minimality~\citep{pearl:09:causality} and SGS-minimality~\citep{spirtes:00:book}. While the former states that from all DAGs that satisfy the causal Markov condition w.r.t.~$P$, the DAG that entails most conditional independence statements is preferred. The latter assumes that no proper subgraph of the true DAG fulfils the causal Markov condition w.r.t.~to $P$. From both assumptions, SGS-minimality is the weaker assumption~\citep{zhang:13:comparison}. 
In a different line of research, it was shown that SGS-minimality suffices for causal discovery approaches based on the additive noise assumption~\citep{peters:14:anm}.

A more recent approach by \cite{forster:17:frugality} introduces the concept of \textit{frugality}, which is a stronger assumption than both minimality assumptions. The authors define a DAG $G$ to be more frugal than $G'$, if $G$ contains fewer edges than $G'$. A maximally frugal DAG uses only as many edges as are necessary to satisfy the causal Markov condition. 
To determine maximally frugal graphs, one has to consider all causal orderings of the variables, which is rather costly, but can be solved using permutation algorithms~\citep{raskutti:18:permutation}. Another approach to discover causal graphs based on frugality, or any of the above assumptions is based on 
\iflong
\else
SAT solvers~\citep{zhalama:17:sat}.
\fi

\iflong
\else
In this paper, we introduce $2$-adjacency faithfulness, which allows us to find xor-type relations, some faithfulness violations induced by cancelling paths and all relations that are detectable by assuming adjacency faithfulness.
\fi
We conjecture that $2$-adjacency faithfulness in combination with some minimality assumption, e.g. SGS minimality (see Section~\ref{sec:towards}), is a slightly stronger assumption than frugality since frugality considers all permutations~\citep{forster:17:frugality}. Hence, frugality might be able to also detect structures for which it is necessary to observe more than three nodes to find a dependence. However, how often such structures occur in real data is unknown. Thus, resorting to only consider all triple structures might be more efficient than having to check all permutations.
In addition, we extend existing work by providing a sound orientation rule that can be used to infer the edges within a $2$-association, if they appear in a larger graph.

\iflong
\else
Next, we discuss xor-type relations in more detail.
\fi
We use those structures as an example to illustrate one of the main properties of $2$-associations, that we describe in Theorem~\ref{th:collider}.

\section{UNFAITHFUL TRIPLES}
\label{sec:ufts}

We first define what we call an unfaithful triple\footnote{\cite{ramsey:06:cpc} used the term unfaithful triple for the non-detectable faithfulness violation explained in Section~\ref{sec:cancelling:path}.} and its properties, and then argue why such a triple a) violates adjacency faithfulness and b) even if detected, the underlying DAG structure cannot be uniquely determined without further information.

\begin{definition}[Unfaithful Triple]
\label{def:unfaithful:triple}
Given $\Models = (G, \variables, P)$ and three distinct nodes $X,Y,Z \in \variables$: if $X,Y$ and $Z$ are marginally independent but not mutually independent in $P$, we call $\{ X,Y,Z \}$ an unfaithful triple w.r.t.~$P$.\!\footnote{Not mutually independent implies that $X \nIndep_P \{ Y, Z \}$, $Y \nIndep_P \{ X, Z \}$ or $Z \nIndep_P \{ X, Y \}$.} If further for each distinct pair of nodes $A, B \in \{ X,Y, Z \}:$
\[
\forall \bm{S} \subseteq \variables \backslash \{ X,Y,Z \}: A \nIndep_P B \mid \bm{S} \cup \{ X,Y, Z \} \backslash \{ A,B \} \, ,
\]
we call $\{X,Y,Z \}$ a minimal unfaithful triple.
\end{definition}

The first example for such a triple for three binary random variables was given by \cite{bernstein:27:theory}, which is equivalent to 
\iflong
\else
Example~\ref{ex:one}.
\fi
The minimality condition ensures that the three nodes are connected by a path of length two, as we will show below. 
This concept is also illustrated in Figure~\ref{fig:minimal-uft}.

We start by showing that if three random variables $\{ X,Y,Z \}$ are marginally independent, finding a dependence between all three variables, e.g.~$X \nIndep_P \{ Y, Z \}$, implies that also $Y \nIndep_P \{ X, Z \}$ and $Z \nIndep_P \{ X, Y \}$.
\begin{restatable}{lemma}{leuftand}
\label{le:uftand}
Given $\Models = (G, \variables, P)$, let $\{ X,Y,Z \} \subseteq \variables$ form an unfaithful triple in $P$, then $X \nIndep_P \{ Y, Z \}$, $Y \nIndep_P \{ X, Z \}$ and $Z \nIndep_P \{ X, Y \}$, which in addition implies that $X \nIndep_P Y \mid Z$, $X \nIndep_P Z \mid Y$ and $Y \nIndep_P Z \mid X$.
\end{restatable}
\iflong
\else
We provide the proofs for Lemma~\ref{le:uftand} and the following Lemma~\ref{le:connected} and Theorem~\ref{th:collider} in Supplementary Material~\ref{sup:proofs}.
\fi

\begin{figure}[t]%
	\centering
	\includegraphics[]{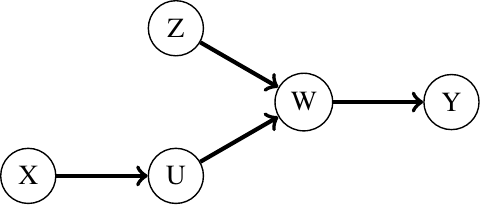}
	\caption{Assume that $\{ X,Y,Z \}$ form an unfaithful triple. Since $X$ is $d$-separated from $Y$ given $U$ and $Z$, they do not form a minimal unfaithful triple. Neither do $\{ U ,Y ,Z \}$, since $Y$ can be $d$-separated from $U$ given $\{ W, Z \}$. Thus, only $\{ U,W, Z \}$ can be a minimal unfaithful triple.}
	\label{fig:minimal-uft}
\end{figure}

Consider Example~\ref{ex:one}. Since $X,Y,Z$ form an unfaithful triple, we can infer from Lemma~\ref{le:uftand} that each pair is conditionally dependent given the third node. As there are no other nodes in the graph, $X,Y,Z$ must form a minimal unfaithful triple.

Next, we show that (minimal) unfaithful triples must be connected in the causal graph.

\begin{restatable}{lemma}{leconnected}
\label{le:connected}
Given $\Models = (G, \variables, P)$, let $\{ X,Y,Z \} \subseteq \variables$ form an unfaithful triple in $P$. If CMC holds, each node in the triple is $d$-connected to at least one other node in the triple by a path in $G$.
\end{restatable}

Further, we show that a minimal unfaithful triple has to contain a collider on a path of length two that connects all three nodes in the triple, e.g. $X \to Y \leftarrow Z$. To do that, we first show a more general statement.

\begin{restatable}{theorem}{thcollider}
\label{th:collider}
Given $\Models = (G, \variables, P)$ with distinct $X,Y,Z \in \variables$ and assume that CMC holds. If $\forall \bm{S} \subseteq \variables \backslash \{ X,Y,Z \}$ it holds that $X \nIndep_P Y \mid Z \cup \bm{S}$, $X \nIndep_P Z \mid Y \cup \bm{S}$ and $Y \nIndep_P Z \mid X \cup \bm{S}$, then one of the three nodes is a collider on a path of length two between the two other nodes, e.g. $X \to Y \leftarrow Z$ in $G$.
\end{restatable}

The theorem only states that there exists a collider, e.g. $X \to Y \leftarrow Z$, but not whether this path is shielded or not. Since we do not assume any marginal dependence or independence in Theorem~\ref{th:collider}, we can derive that the same statement holds for a minimal unfaithful triple. Notice that for a minimal unfaithful triple each pair of nodes is marginally independent, which implies that there is no way to decide which of the three possible collider structures corresponds with the causal graph in the absence of further information.

Knowing that a minimal unfaithful triple has to contain a collider in $G$, it is obvious that such a structure violates adjacency faithfulness, as none of the edges is represented by a marginal dependence in $P$. 
The key point is that we can detect such interactions by taking multiple parents into account. In the following, we define a weaker assumption that allows us to detect and infer causal graphs that contain such faithfulness violations.

\section{$2$-ADJACENCY FAITHFULNESS}
\label{sec:connections}

To define our new assumption, we first need to define associations between a single node and a set of nodes.

\begin{definition}[$k$-Association]
\label{de:k-associated}
Let $P$ be the joint distribution of a set of observed random variables $\variables$. 
\begin{enumerate}
	\item Given distinct $X,Y \in \variables$, we say that X is $1$-associated to $Y$, if $\forall \bm{S} \subseteq \variables \backslash \{X,Y \}: X \nIndep_P Y \mid \bm{S}$.
	\item Given distinct $X, Y_1,Y_2 \in \variables$, $X$ is $2$-associated to $\{ Y_1,Y_2 \}$ if $\forall \bm{S} \subseteq \variables \backslash \{X,Y_1,Y_2 \}$ it holds that
	\begin{enumerate}[label=\roman*)]
		\item $X \nIndep_P Y_1 \mid \bm{S} \cup Y_2$,
		\item $X \nIndep_P Y_2 \mid \bm{S} \cup Y_1$ and
		\item $Y_1 \nIndep_P Y_2 \mid \bm{S} \cup X$.
	\end{enumerate}
\end{enumerate}
We call $X$ strictly $2$-associated to $\{ Y_1,Y_2 \}$, if $X$ is $2$-associated to $\{ Y_1, Y_2 \}$ and not $1$-associated to $Y_1$ or $Y_2$.
\end{definition}

In other words, $k$-associations relate to two types of dependencies: certain conditional dependencies between pairs of variables ($1$-associations) and between triples ($2$-associations). For readability, we use a shorthand notation and write $X \atwo \{ Y,Z \}$ if $X$ is $2$-associated to $Y$ and $Z$ resp. $X \aone Y$ if $X$ is $1$-associated to $Y$. We denote a strict $2$-association by ``$\atwos$''. If we refer to a set $\bm{Y}$ that contains at most two elements and we want to express that $X$ is either $1$- or $2$-associated to this set, we write $X \aletwo \bm{Y}$. Similarly, we write $X \aletwos \bm{Y}$, if $X$ is $1$- or strictly $2$-associated to $\bm{Y}$.

Pairwise dependencies can occur for example in a simple chain $X \to Y \to Z$, where no adjacency failure occurs. In this case, $X \aone Y$ and $Y \aone Z$. Triple interactions that match the definition of $2$-associations, however, need to have a specific structure. 
\iflong
\else
As we saw in Theorem~\ref{th:collider}, $2$-associations always contain a collider.
\fi
Thus, a chain graph or a common cause structure does not induce a $2$-association. On the other hand, the minimum unfaithful triple in Example~\ref{ex:one} matches the definition, since $X \atwos \{ Y, Z \}$, $Y \atwos \{ X, Z \}$ and $Z \atwos \{ X, Y \}$. 
In general, strict $2$-associations describe collider structures such as $X \to Y \leftarrow Z$ for which at least one of the edges violates adjacency faithfulness. If faithfulness holds, a collider structure induces a $2$-association, but not a strict $2$-association. We use this intuition for our new assumption.

\begin{definition}[$2$-Adjacency Faithfulness]
Given $\Models = (G, \variables, P)$, for all $X, Y \in \variables$, where $X$ and $Y$ are adjacent in $G$, there exists $\bm{Y} \subseteq \MB(X)$, with $Y \in \bm{Y}$, s.t. $X \aletwos \bm{Y}$.
\end{definition}

The main idea here is to weaken adjacency faithfulness such that if a marginal dependence is not present, i.e., adjacency faithfulness is violated, there will be a dependence in combination with a parent, child or spouse. If adjacency faithfulness is not violated, we will not find any strict $2$-associations and our assumption reduces to adjacency faithfulness. By also considering strict $2$-associations, however, we can discover a larger spectrum of causal mechanisms.

The textbook example for a mechanism that violates faithfulness but is detectable by assuming $2$-adjacency faithfulness is the xor-connection described in Example~\ref{ex:one}. Here, $Y \atwos \{ X, Z \}$, two parents, while $X \atwos \{Y, Z \}$---i.e. a child and a spouse. We could even slightly adapt the mechanism and only model $Z$ using an unbiased coin but use a biased coin for $X$. In this case, only $X$ is marginally independent of $Y$ and $Z$, while $Z$ becomes dependent on $Y$. 

Moreover, assuming $2$-adjacency faithfulness could even allow us to detect some faithfulness violations that are due to cancelling paths. In particular, consider the two paths $X \to Y$ and $X \to Z \to W \to Y$ that cancel such that $X \Indep_P Y$. Since $X \Indep_P Y$, $X \Indep_P W \mid Z$,
\iflong
\else
$X$ is not $1$-associated to $Y$ or $W$ and hence could be strictly $2$-associated to the set $\{ W,Y \}$. 
\fi
Since we know that a $2$-association contains a collider and we can neither find a $1$-association to $Y$ or $W$, we know that there has to be an edge violating adjacency faithfulness. 

It is not possible to rely on orientation faithfulness when dealing with strict $2$-associations. Although we know that a strict $2$-association has to contain a collider, we do not know the skeleton structure within the triple and hence cannot apply orientation faithfulness. 

\section{ORIENTATION}
\label{sec:orientation}

So far, we showed how we can detect unfaithful triples from conditional (in)dependence statements under the weaker assumption of $2$-adjacency faithfulness. Now imagine that we want to use this knowledge for causal discovery. If we observe an isolated triple that follows the dependence structure of the noisy xor, we can only tell that there is a collider.
However, if we are given more information, we are able to break this symmetry. 

\begin{figure}[t]%
	\begin{minipage}[t]{.5\linewidth}
		\centering
		\includegraphics[]{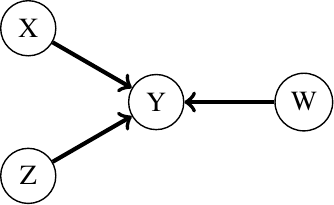}
		\subcaption{}
	\end{minipage}%
	\begin{minipage}[t]{.5\linewidth}
		\centering
		\includegraphics[]{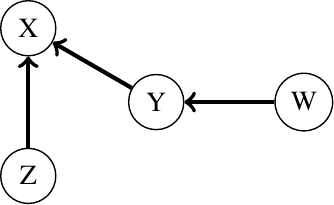}
		\subcaption{}
	\end{minipage}%
	\caption{In both distributions $Y \atwos \{X, Z \}$ and $Y-_1 W$. In the graph shown in (a) $Y$ is a collider on all paths between $\{X,Z \}$ and $W$, whereas in (b) $Y$ is a non-collider.}
	\label{fig:2-orientation-ff}
\end{figure}

\begin{example}
\label{ex:two}
Consider that $X$ and $Z$ are unbiased coins as in the noisy xor example. In addition, there is a binary variable $W$ with $P(W=1) = p$, where $0 < p < 1$ and an unobserved binary noise variable $E$ with $P(E=1) = q$, where $0 < q < \frac{1}{2}$. Now we generate $Y$ as
\[
Y := ((X \oplus Z) \land W) \oplus E \; ,
\]
where $E,W,X$ and $Z$ are drawn independently. The requirements for $p$ ensure that $W$ is dependent on $Y$ and the requirements on $E$ ensure that the dependencies are non-deterministic ($q \neq 0$) and evident without observing $E$ ($q \neq \frac{1}{2}$). The corresponding causal graph is given in Figure~\ref{fig:2-orientation-ff}(a). From the induced dependencies, that we derive in detail in Supplementary Material~\ref{sup:example}, we can now obtain an asymmetry. In particular, $\{ X,Y,Z \}$ form a minimal unfaithful triple, but only $Y$ is dependent on $W$, whereas $\{ X,Z \} \Indep_P W$ and due to the xor, $X \Indep_P W \mid Y$ as well as $Z \Indep_P W \mid Y$. Thus, we can detect that there is no edge between $X$ and $W$ or $Z$ and $W$ since none of these pairs can be $2$-associated. However, we do find that $X \nIndep_P W \mid \{ Y,Z \}$ and $Z \nIndep_P W \mid \{ Y,X \}$. As we will show in Theorem~\ref{th:soundness}, we can use this information to identify that $Y$ is the collider in the triple and that $W \to Y$.
\end{example}

To detect such an asymmetry, it is necessary that the collider in the triple is the effect of another node or pair of nodes. If, for example, $X$ would be the collider in the triple and $W \to Y$ (see Figure~\ref{fig:2-orientation-ff}(b)), we cannot find such an asymmetry. To generate that graph we could model $Y$ as a noisy copy of $W$ and construct $X$ with a noisy xor from $Y$ and $Z$. We still know that $W$ is adjacent to $Y$, but we cannot direct any of the edges as for example we would find that $X \nIndep_P W \mid Z$, which we would also observe if $Z$ would be the collider in the triple, or if we would flip the edge direction 
\iflong
\else
between $Y$ and $W$---i.e.~if $W$ is a noisy copy of $Y$.
\fi

Based on this intuition, we propose an orientation rule that may include causal structures that induce strict $2$-associations. To do so, we use a shorthand notation---i.e. write $\bm{Y} \rightarrow X$, if for each element $Y \in \bm{Y}$ it holds that $Y \rightarrow X$ and vice versa write $X \rightarrow \bm{Y}$ if $X$ is a parent of each node $Y \in \bm{Y}$, that is, $\forall Y \in \bm{Y}: X \to Y$.

\begin{definition}[Orientation Rule]
\label{def:2orientationrule}
Let $M := (G,\variables,P)$ and we are given two disjoint sets $\bm{X}, \bm{Z} \subseteq \variables$ and $Y \in \variables$, where $Y \aletwos \bm{X}$ and $Y \aletwos \bm{Z}$, and no node $X \in \bm{X}$ is adjacent to some node $Z \in \bm{Z}$.
\begin{enumerate}[label=\roman*)]
	\item If for each pair $X \in \bm{X}$ and $Z \in \bm{Z}$ it holds that $X$ is dependent on $Z$ given any subset of $\variables \backslash \{ X,Z \}$ that contains ${Y} \cup (\bm{X} \backslash \{X \}) \cup (\bm{Z} \backslash \{Z \})$, then $\bm{X} \to Y \leftarrow \bm{Z}$,
	\item otherwise, if for each pair $X \in \bm{X}$ and $Z \in \bm{Z}$ it holds that $X$ is dependent on $Z$ conditional on any subset of $\variables \backslash \{ X,Z \}$ that contains $(\bm{X} \backslash \{X \}) \cup (\bm{Z} \backslash \{Z \})$ but does not contain $Y$, $Y$ is a non-collider on at least one path $\langle X,Y,Z \rangle$ where $X \in \bm{X}$ and $Z \in \bm{Z}$.
\end{enumerate}
\end{definition}

Simply put, the above orientation rule relies on the fact that a (strict) $2$-association contains a collider. Either $Y$ is the collider on each path $\langle X,Y,Z \rangle$ between any variable $X \in \bm{X}$ and $Z \in \bm{Z}$ or $Y$ is one of the parents in at least one of the triples and hence blocks at least one such path.
If both sets $\bm{X}$ and $\bm{Z}$ only contain a single element, rule i) refers to a ``normal'' collider e.g. $X \to Y \leftarrow Z$ and rule ii) refers either to a chain like $X \to Y \to Z$ or to a common cause $X \leftarrow Y \to Z$. Let us consider Example~\ref{ex:two} again, where we generated $Y$ as a non-deterministic function of $X,Z$ and $W$.
First, we find that $Y \atwos \{X,Z \}$, $Y \aone W$ and $W$ is not adjacent to $X$ or $Z$ (since $W$ is not $1$- or strictly $2$-associated to $X$ or $Z$), which is required to apply our rule. Further, we can apply rule i) since  $W$ is dependent on $X$ given any set that includes $\{Y, Z \}$ and $W$ is dependent on $Z$ given any set that includes $\{Y,X \}$. Hence, 
\iflong
\else
we can infer the correct DAG structure $\{ X,Z \} \to Y \leftarrow W$.
\fi

In the following we will first show that our orientation rule is sound---i.e. if rule i) or ii) can be applied, then we are sure we found the corresponding graph structure---and then analyze the inverse, that is, what assumptions need to hold s.t. the given graph implies the suggested dependence model.

\begin{restatable}{theorem}{thsoundness}
\label{th:soundness}
Assuming that the causal Markov condition holds, the orientation rule in Definition~\ref{def:2orientationrule} is sound.
\end{restatable}

\iflong
\else
We provide the proof in Supplementary Material~\ref{sup:proofs}.
\fi

The question that remains is: Does the inverse always hold? For example, if the true graph contains a non-collider structure such as $X \to Y \to Z$, will we always find that $X \nIndep_P Z$? The short answer is no. Already when we only assume adjacency faithfulness, it can happen that $X \Indep_P Z$ although the true graph is $X \to Y \to Z$ and it holds that $X \nIndep_P Y$ and $Y \nIndep_P Z$, which is called failure of transitivity. More generally, assuming that orientation faithfulness holds, such failures will not occur. In the following, we extend this assumption to our setting.

\begin{definition}[2-Orientation Faithfulness]
\label{def:2orientationff}
Let $M := (G,\variables,P)$ and we are given two disjoint sets $\bm{X}, \bm{Z} \subseteq \variables$ and $Y \in \variables$, where $Y \aletwos \bm{X}$ and $Y \aletwos \bm{Z}$, and no node $X \in \bm{X}$ is adjacent to some node $Z \in \bm{Z}$.
\begin{enumerate}[label=\roman*)]
	\item If $\bm{X} \to Y \leftarrow \bm{Z}$ is in $G$, then for each pair $X \in \bm{X}$ and $Z \in \bm{Z}$, $X$ is dependent on $Z$ given any subset of $\variables \backslash \{ X,Z \}$ that contains ${Y} \cup (\bm{X} \backslash \{X \}) \cup (\bm{Z} \backslash \{Z \})$,
	\item otherwise, for each pair $X \in \bm{X}$ and $Z \in \bm{Z}$, $X$ is dependent on $Z$ conditional on any subset of $\variables \backslash \{ X,Z \}$ that contains $(\bm{X} \backslash \{X \}) \cup (\bm{Z} \backslash \{Z \})$, but not $Y$.
\end{enumerate}
\end{definition}

Equivalently to $2$-adjacency faithfulness, $2$-orientation faithfulness reduces to orientation faithfulness, if both sets $\bm{X}$ and $\bm{Z}$ only contain a single element. For orientation faithfulness, it has been shown that all failures can be detected under the assumption that adjacency faithfulness holds~\citep{zhang:08:triangle}. Sadly, an equally strong statement cannot be made for $2$-adjacency faithfulness and $2$-orientation faithfulness, as we discuss 
\iflong
\else
below.
\fi

\subsection{Failures of $2$-Orientation Faithfulness}

\begin{figure}[t]%
	\begin{minipage}[t]{.5\linewidth}
		\hspace*{-12pt}
		\includegraphics[]{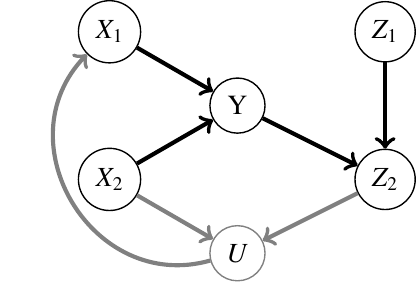}
		\subcaption{}
	\end{minipage}%
	\begin{minipage}[t]{.5\linewidth}
		\hspace*{-12pt}
		\includegraphics[]{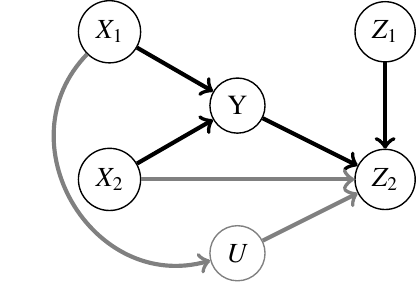}
		\subcaption{}
	\end{minipage}%
	\caption{In both figures, $Y \atwos \bm{X} = \{X_1, X_2 \}$, $Y \atwos \bm{Z} = \{Z_1, Z_2 \}$ (related nodes and edges are marked in black) and $X_2 \atwos \{ U, Z_2 \}$. 
	\iflong
\else
	Given only this information, we cannot tell if the path $\langle X_2,Y,Z_2 \rangle$ is unshielded (a) or shielded (b). 
	\fi
	While in graph (a), we could safely apply our orientation rule, the shielded graph (b) can be problematic. Due to the directed path from $X_1$ over $U$ to $Z_2$ and the adjacency between $X_2$ and $Z_2$, each pair $X,Z \in \bm{X} \times \bm{Z}$ is now $d$-connected given $\{ Y \} \cup (\bm{X} \backslash \{ X \}) \cup (\bm{Z} \backslash \{ Z \})$. Thus, the condition for rule i) could hold, although $\bm{X} \to Y \leftarrow \bm{Z}$ is not in $G$.}
	\label{fig:unshielded-shielded}
\end{figure}

Without any assumptions, we can detect triples for which $Y \aletwos \bm{X}$ and $Y \aletwos \bm{Z}$, and know by assuming CMC that all $2$-associations contain a collider. If further, all paths $\langle X,Y,Z \rangle$ with $(X,Z) \in \bm{X} \times \bm{Z}$ are unshielded, we can detect if any of the conditions in $2$-orientation faithfulness fails. In particular, due to the soundness of our orientation rule, we would detect that none of the conditions in the orientation rule is satisfied if condition i) or ii) in $2$-orientation faithfulness fails, as we show in Corollary~\ref{cor:detection}.

Yet, we cannot detect all failures of $2$-orientation faithfulness. That is, due to the fact that we might not always be able to detect whether all paths $\langle X,Y,Z \rangle$ are unshielded. If there is a direct edge between $X$ and $Z$, we will always find that those nodes are either $1$-associated or there exists a third node $U$ such that they are strictly $2$-associated (if $2$-adjacency faithfulness holds). However, if we find a strict $2$-association between $X$ and $\{ Z,U \}$ there is no guarantee that the path is shielded. In particular, if $U$ is the collider between $X$ and $Z$, the triple is unshielded; but if $Z$ is the collider between $X$ and $U$, the triple is shielded (see Figure~\ref{fig:unshielded-shielded}, in which $X$ refers to $X_2$ and $Z$ to $Z_2$). In a causal discovery algorithm, we could try to iteratively infer the DAG structure within such triples until we cannot apply the rule anymore. If we are lucky, we can first infer that $X \to U \leftarrow Z$ and after that also apply our rule for $\{ \bm{X},Y,\bm{Z} \}$. Keeping this exception in mind, we can derive the following corollary from Theorem~\ref{th:soundness}.

\begin{restatable}{corollary}{cordetection}
\label{cor:detection}
Given $M := (G,\variables,P)$ with $Y \in \variables$ and $\bm{X}, \bm{Z} \subseteq \variables$, where $\bm{X} \cap \bm{Z} = \emptyset$, $Y \aletwos \bm{X}$, $Y \aletwos \bm{Z}$ and no pair of nodes $(X,Z) \in \bm{X} \times \bm{Z}$ is adjacent. Assuming that CMC holds, we can detect if condition i) or ii) of $2$-orientation faithfulness fails on the triple $\{ \bm{X}, Y, \bm{Z} \}$.
\end{restatable}

The proof is provided in Supplementary Material~\ref{sup:proofs}. In general, $2$-orientation faithfulness might be useful not only for constraint-based causal discovery methods, but also for algorithms that aim to discover the Markov blanket of a node or permutation-based discovery algorithms such as the Sparsest Permutation (SP) algorithm proposed by \cite{raskutti:18:permutation}. In Supplementary Material~\ref{sup:smr}, we provide a short discussion from which we conjecture that the SP algorithm can identify the collider pattern for strict $2$-associations as in Figure~\ref{fig:2-orientation-ff}, if $2$-orientation faithfulness holds.

\section{IMPLEMENTATION}
\label{sec:implementation}

As a proof of concept, we propose a simple modification of the Grow and Shrink (GS) algorithm~\citep{margaritis:00:gs} to discover Markov blankets that may contain strict $2$-associations. After that, we briefly discuss further challenges that need to be solved to propose a causal discovery algorithm based on our new assumptions.

The GS algorithm is a simple and theoretically sound causal discovery algorithm, that as a first step identifies the Markov blanket for each node~\citep{margaritis:00:gs}. This step of the algorithm consists of a grow phase, in which we iteratively discover a superset of the Markov blanket of a target node $T$, and a shrink phase, in which superfluous nodes are pruned.

To make sure that we can detect Markov blankets that contain strict $2$-associations, we assume that $2$-adjacency faithfulness holds and that we can detect all spouses. For the latter, there are two options. Either the target node, the spouse and the common child are connected via a strict $2$-association, or the spouse node is only $1$- or strictly $2$-associated with the common child and not with the target node. For the first option, it suffices to assume $2$-adjacency faithfulness to detect the spouse, whereas for the second option we need to assume a variant of $2$-orientation faithfulness. In particular, consider the graph $T \to C \leftarrow S$, in which $T \aone Y$ and $C \aone Y$. Assuming only $2$-adjacency faithfulness will not guarantee that $T \nIndep S \mid C$, additionally assuming $2$-orientation faithfulness will. More generally, we need to assume that condition i) in $2$-orientation faithfulness also holds for shielded triples, which boils down to assuming that the spouses of the target do not cancel each other out, as we explain below.

\begin{assumption}
\label{as:one}
Let $M := (G,\variables,P)$ and we are given two disjoint sets $\bm{X}, \bm{Z} \subseteq \variables$ and $Y \in \variables$, where $Y \aletwos \bm{X}$ and $Y \aletwos \bm{Z}$. If $\bm{X} \to Y \leftarrow \bm{Z}$ in $G$, then for each pair $X \in \bm{X}$ and $Z \in \bm{Z}$, $X$ is dependent on $Z$ given any subset of $\variables \backslash \{ X,Z \}$ that contains ${Y} \cup (\bm{X} \backslash \{X \}) \cup (\bm{Z} \backslash \{Z \})$.
\end{assumption}

The above assumption is a relatively lightweight adaption of condition i) in $2$-orientation faithfulness.
In particular, let $\bm{X} = \{ X,T \}$, where $T$ is the target node. Then all nodes in $\bm{Z}$ are spouses of $T$, which we can detect if Assumption~\ref{as:one} holds and even become part of $\PC(T)$ if all paths $\langle T,Y, Z \rangle$ for $Z \in \bm{Z}$ are shielded. Thus, we would already add those nodes when looking for the parents and children of $T$. The only complication that may arise is if the second node $X \in \bm{X}$ is adjacent to a node in $Z \in \bm{Z}$ and this adjacency would lead to a cancellation such that $Z$ is only dependent on $T$ if we do not condition on $X$. The corresponding causal graph consists of the paths $T \to Y \leftarrow Z$ and $Y \leftarrow X \to Z$. 
\iflong
\else
Since $X$ cannot block the path $\langle T,Y,Z \rangle$, such a scenario seems only possible if the causal mechanism that generates $Z$ from $X$ is deterministic. Based on this assumption, we introduce our adapted GS algorithm.
\fi

The generalized GS algorithm is shown in Algorithm~\ref{alg:alg1}, where we only modified the grow phase to also consider pairs of random variables. This allows us to find nodes to which the target node is strictly $2$-associated or spouses to which a child node of $T$ is strictly $2$-associated using Assumption~\ref{as:one}.
\iflong
\else
The shrink phase is not modified and checks if singletons can be removed. Importantly, we will not remove single nodes of a true strict $2$-association to $T$ or a child of $T$, because we do not check for marginal dependencies. 
\fi
For example, assume that $T \atwos \{ X, Z \}$ and both nodes were added in the grow phase, where $X$ is a child of $T$ and $Z$ the corresponding spouse. If we try to remove $X$ in the shrink phase, we find that $T \nIndep_P X \mid \bm{S} \backslash X$, since $Z \in \bm{S}$. Hence, $X$ remains in $\bm{S}$, as well as $Z$.

\begin{algorithm}[t!]
	\caption{Modified GS for Markov Blankets}
	\label{alg:alg1}
	\Input{Random variables $\variables$ with joint distribution $P$, Target $T \in \variables$}
	\Output{$\MB(T)$}
	$\variables' \leftarrow \variables \backslash \{ T \} $; \\
	$\bm{S} \leftarrow \emptyset$; \\
	\tcp{Grow Phase}
	\While{$\left( \exists X \in \variables': T \nIndep_P X \mid \bm{S} \right) \; \lor$ \\
	$\left( \exists X,Z \in \variables': T \nIndep_P X \mid \bm{S} \cup \{ Z \} \right)$}{
		$\bm{S} \leftarrow \bm{S} \cup \{ X \}$  \;
	}
	\tcp{Shrink Phase}
	\While{$\exists X \in \bm{S}: T \Indep_P X \mid \bm{S} \backslash X$}{
		$\bm{S} \leftarrow \bm{S} \backslash X$ \;
	}
	\Return{$\bm{S}$}
\end{algorithm}

In the following, we show that our proposed algorithm correctly identifies the Markov blanket of a target node assuming that $2$-adjacency faithfulness, the causal Markov condition and Assumption~\ref{as:one} hold.

\begin{restatable}{theorem}{thcorrectness}
\label{th:correctness}
Given $M = (G,\variables,P)$. Assuming that $2$-adjacency faithfulness, Assumption~\ref{as:one} and CMC hold, Algorithm~\ref{alg:alg1} correctly identifies $\MB(T)$ for $T \in \variables$.
\end{restatable}

We provide the proof in Supplementary Material~\ref{sup:proofs}. For discovering the Markov blanket, we do not need to know the collider of a strict $2$-association since it only returns a set of nodes. The more challenging task is to implement our framework to discover causal networks, which we will briefly discuss below.

\subsection{TOWARDS CAUSAL NETWORK INFERENCE}
\label{sec:towards}

In this paper, we mainly focused on answering two questions: 1) How can we weaken faithfulness to detect xor-type structures, and 2) under which conditions can we identify the collider in such a triple? The assumptions that we derived are sufficient for Markov blanket discovery, as we showed above. However, in Markov blanket discovery, we are not concerned with finding a unique graph structure, our only goal is to detect the set of nodes, which contains the parents, children and spouses of a target node.

For causal discovery, assuming $2$-adjacency faithfulness might be too inclusive. To illustrate this statement, consider Example~\ref{ex:one} again, in which $X$ and $Z$ cause $Y$ through a noisy xor. Besides the three possible collider structures $X \to Y \leftarrow Z$, $Y \to X \leftarrow Z$ and $X \to Z \leftarrow Y$, also a fully connected graph, e.g. $X \to Y \leftarrow Z$ and $X \to Z$, is compatible with $2$-adjacency faithfulness since $X$ is not $1$-associated to $Z$ or $Y$. To avoid finding such graphs with superfluous edges, we need to restrict the search space to those graphs $G'$ for which no proper subgraph is compatible with $2$-adjacency faithfulness and $2$-orientation faithfulness. For future work, we want to investigate how this can be achieved. One possibility could be to combine our assumptions with SGS-minimality~\citep{spirtes:00:book}, which assumes that no proper subgraph of the true DAG $G$ entails the causal Markov condition w.r.t. to $P$. In Example~\ref{ex:one}, assuming SGS-minimality in addition to $2$-adjacency faithfulness would reduce the number of admissible graphs to the three possible collider structures, which include the true DAG.

To derive a causal discovery algorithm under the above assumptions, the most straightforward approach would be to further extend the GS algorithm. After detecting all Markov blankets, the GS algorithm distinguishes the spouses of a Markov blanket from the parents and children by detecting collider patterns. In this step, we could extend the existing rule with a modification of Definition~\ref{def:2orientationrule}. Similarly, we could extend well-known algorithms such as the PC algorithm~\citep{spirtes:00:book} or the GES algorithm~\citep{chickering:02:ges} by modifying the skeleton phase, respectively the forward phase such that we can find triple interactions as we did for GS. The edge orientation could be done by first applying the orientation rule in Definition~\ref{def:2orientationrule} and then applying a similar set of rules like Meek's orientation rules~\citep{meek:95:orientation}. Alternatively, it was shown that SAT-based causal discovery algorithms can be easily adapted to weaker assumptions than faithfulness~\citep{zhalama:17:sat}, which might be an interesting direction for future work.

\section{CONCLUSION}
\label{sec:conclusion}

In this work, we proposed $2$-adjacency faithfulness, which is a weaker version of adjacency faithfulness. Our new assumption is able to detect faithfulness violations caused by weak or non-existent marginal dependencies, which are detectable by considering a combination of parents, children or spouses. We provide an in-depth analysis of such dependencies and propose a sound orientation rule, which can infer part of the correct causal structure by detecting colliders. We complement this rule with $2$-orientation faithfulness, which assumes that if a causal graph contains such collider structures, we will find that the corresponding conditional dependence statements hold in $P$. As a proof of concept, we extended the GS algorithm to find Markov blankets under strictly weaker assumptions than faithfulness.

For future work, we would like to develop a sound causal discovery algorithm based on $2$-adjacency faithfulness and extend our theory to directed mixed graphs\iflong\else.\fi

\begin{acknowledgements} 
The authors would like to thank the anonymous reviewers for insightful and valuable comments.
This work was supported by the European Research Council (ERC) under the European Union's Horizon 2020 research and innovation programme (grant agreement 639466).
A. Marx is supported by the International Max Planck Research School for Computer Science (IMPRS-CS).
\end{acknowledgements} 


\bibliography{bib/abbreviations,bib/bib-paper,bib/bib-alex}

\appendix

\section*{SUPPLEMENTARY MATERIAL}

\setcounter{section}{19}

\subsection{Example~2 in Detail}
\label{sup:example}

As described in Section~\ref{sec:orientation}, we can generate a DAG of the form $X \to Y \leftarrow Z$ and $W \to Y$ s.t. $X,Y$ and $Z$ form a minimal unfaithful triple and $W \nIndep_P Y$ as follows. We generate $X, Z, W$ and $E$ independently, with $X$ and $Z$ as fair coins, $W$ as a coin with $P(W = 1) = p$, where $0 < p < 1$ and $E$ (the noise variable) as a biased coin with $P(E = 1) = q$, $0 < q < \frac{1}{2}$. With $q > 0$, we ensure that the function is non-deterministic. Further, we generate $Y$ as
\[
Y := ((X \oplus Z) \land W ) \oplus E \; .
\]
We will obtain that $P(Y = 1) = q + \frac{p}{2} - pq$. Further, we can calculate that $P(X=1 , Y= 1) = \frac{1}{2} P(Y = 1) = P(X=1) \cdot P(Y = 1)$. Also, $P(X=1,Y=0) = P(X=1) \cdot P(Y = 0)$, which means that they are marginally independent. The same holds for $Z$ and $Y$. If we calculate the probability for all three variables, we get that $P(X = 0, Z=1, Y=1) = \frac{p + q - 2pq}{4}$ and $P(X=0,Z=1) \cdot P(Y=1) = \frac{1}{4} P(Y = 1)$. Hence, we need to solve
\begin{align}
P(X = 0, Z=1, Y=1) &= P(X=0,Z=1) \cdot P(Y=1) \\
\Leftrightarrow p + q - 2pq &= q + \frac{p}{2} - pq \\
\Leftrightarrow  p - pq &= \frac{p}{2} \; .
\end{align}
The only solutions are $p=0$ or $q= \frac{1}{2}$, which we excluded. Hence, $Y \nIndep_P \{ X,Z \}$ and by weak union also $Y \nIndep_P X \mid Z$, as well as $Y \nIndep_P Z \mid X$. Since we know by assumption that $X \Indep_P Z$ we can conclude from Lemma~\ref{le:uftand} that also $X \nIndep_P Z \mid Y$, which means that $\{ X,Y,Z \}$ from a minimal unfaithful triple since $W$ will also not cancel out any of these conditional dependencies. Next, we also find that $W \nIndep_P Y$, since $P(W=1,Y=1) = \frac{p}{2}$, which is only equal to $P(W=1) \cdot P(Y=1)$, if $p=0$, $p=1$ or $q= \frac{1}{2}$, which we excluded, and hence $W \nIndep_P Y$. Last, we need to show that $X \nIndep_P W \mid \{Y,Z \}$ and that $Z \nIndep_P W \mid \{X,Y \}$. We can write
\[
P(X,W \mid Y,Z) = \frac{P(X, W, Y,Z)}{P(Y,Z)} \; .
\]
To show conditional dependence, this value has to be different from $P(X \mid Y,Z) \cdot P(W \mid Y,Z)$. Consider the case where all variables are equal to one.
Hence, we get that
\begin{align}
P(X=1, W=1, Y=1,Z=1) &= \frac{pq}{4} \; , \\
P(X=1, Y=1,Z=1) &= \frac{q}{4} \; , \\
P(W=1, Y=1,Z=1) &= \frac{p}{4} \; .
\end{align}
Since we know that $P(Y=1,Z=1) = P(Y=1) / 2$, we thus need to solve
\[
pq = \frac{pq}{2 P(Y=1)} \; .
\]
This equation can only be true if $p$ or $q$ = 0, i.e. the system is either independent of $W$ or deterministic, $p=1$ or $q = \frac{1}{2}$, which we all excluded by assumption. Hence, $X \nIndep_P W \mid \{Y,Z \}$. The dependence between $Z$ and $W$ given $X$ and $Y$ can be derived in the same way.

\subsection{$2$-Orientation Faithfulness and Sparsest Markov Representation}
\label{sup:smr}

In this section, we briefly discuss the connection of our new assumptions to approaches based on the sparsest Markov representation (SMR)~\citep{raskutti:18:permutation} which is also referred to as frugality~\citep{forster:17:frugality}, which we discussed in the related work section. A graph $G^*$ satisfies the SMR assumption if every graph $G$ that fulfils the Markov property and is not in the Markov equivalence class of $G^*$ contains more edges than $G^*$. Here we will not discuss 
the SMR assumption in further detail, but focus on the suggested 
\iflong
\else
permutation-based
\fi
causal discovery algorithm under the SMR assumption, which is called the Sparsest Permutation (SP) algorithm.

To explain the SP algorithm, we need to define a DAG $G_\pi$, w.r.t. a permutation $\pi$. A DAG $G_\pi$ consists of vertices $\variables$ and directed edges $E_\pi$, where an edge from the $j$-th node $\pi(j)$ according to permutation $\pi$ to node $\pi(k)$ is in $E_\pi$ if and only if $j < k$ and
\begin{equation}
X_{\pi(j)} \nIndep_P X_{\pi(k)} \mid \{ X_{\pi(1)}, X_{\pi(2)}, \dots, X_{\pi(k-1)} \} \backslash \{ X_{\pi(j)} \} \; ,
\end{equation}
where $X_{\pi(j)}$ refers to the $j$-th random variable according to permutation $\pi$. Based on this definition, the SP algorithm constructs a graph $G_\pi$ for each possible permutation and selects that permutation $\pi^*$ for which $G_{\pi^*}$ contains the fewest edges. This permutation $\pi^*$ is also called minimal or a minimal permutation, if it is not unique. 

Although this procedure might be very slow in practice, it has theoretically appealing properties. In particular, we conjecture that it can identify the collider pattern even if strict $2$-associations are included, if $2$-orientation faithfulness holds. In this work, we will not provide a proof for this conjecture, but give some evidence by discussing the behaviour of the SP algorithm on an example graph.

Consider the graph provided in Figure~\ref{fig:2-orientation-ff}(a) again. For this example, we assume that \variables does not consist of any further vertices than the four shown in the graph. We will show that all permutations $\pi$ that are minimal have in common that $\pi(4) = Y$. W.l.o.g. let $\pi(1) = X, \pi(2) = Z$ and $\pi(3) = W$, then $G_\pi$ only contains the three correct edges, which are:
\begin{align}
\pi(1) &\to \pi(4): X \nIndep_P Y \mid \{ Z,W \} \\
\pi(2) &\to \pi(4): Z \nIndep_P Y \mid \{ X,W \} \\
\pi(3) &\to \pi(4): W \nIndep_P Y \mid \{ X,Z \} \\
\end{align}
and we do not add any superfluous edges, as
\begin{align}
\pi(1) &\to \pi(2): X \Indep_P Z \mid \emptyset \\
\pi(1) &\to \pi(3): X \Indep_P W \mid Z \\
\pi(2) &\to \pi(3): Z \Indep_P W \mid X \; . \\
\end{align}
If we would pick a permutation $\pi'$ in which we flip for example $W$ and $Y$ such that $Y$ is no longer the node assigned to the highest number in the permutation, i.e. $\pi'(3) = Y$ and $\pi'(4) = W$, we will find more edges and thus not a minimal graph anymore. In particular, we get that 
\begin{align}
\pi'(1) &\to \pi'(3): X \nIndep_P Y \mid \{ Z \} \\
\pi'(2) &\to \pi'(3): Z \nIndep_P Y \mid \{ X \} \\
\pi'(3) &\to \pi'(4): Y \nIndep_P W \mid \{ X,Z \} \\
\pi'(1) &\to \pi'(4): X \nIndep_P W \mid \{ Z,Y \} \\
\pi'(2) &\to \pi'(4): Z \nIndep_P W \mid \{ X,Y \} \\
\end{align}
and thus the graph according to this permutation contains two edges more than for permutation $\pi$. The main point is that we are now allowed to condition on $Y$, which opens the paths between $X$ or $Z$ and $W$. Similarly, assume that we put $X$ as the last node and get the order $\pi'(1) = Z, \pi'(2) = W, \pi'(3)= Y$ and $\pi'(4) = X$, for which
\begin{align}
\pi'(1) &\to \pi'(2): Z \Indep_P W \mid \emptyset \\
\pi'(1) &\to \pi'(3): Z \Indep_P Y \mid \{ W \} \\
\pi'(1) &\to \pi'(4): Z \nIndep_P X \mid \{ W,Y \} \\
\pi'(2) &\to \pi'(3): W \nIndep_P Y \mid \{ Z \} \\
\pi'(2) &\to \pi'(4): W \nIndep_P X \mid \{ Z,Y \} \\
\pi'(3) &\to \pi'(4): Y \nIndep_P X \mid \{ Z,W \} \\
\end{align}
and hence, we again find four edges, which is one more than for $\pi$. Also, if $\pi'(1) = Y$, we can use it in the conditional to find a dependence between $X$ and $Z$ and at least one dependence between $X$ or $Z$ and $W$. 
\iflong
\else
Hence, the SP algorithm would infer a correct ordering for this graph.
\fi

An interesting avenue for future work would be to analyze whether it is possible to always detect the collider pattern also in larger graphs and triples that 
\iflong
\else
may be shielded.
\fi

\subsection{Proofs}
\label{sup:proofs}

\iflong
\else
Before we provide the proofs, we state the graphoid axioms~\citep{dawid:79:graphoid,spohn:80:stochastic,geiger:90:identifying}, which are used in several of our proofs.

\begin{definition}[Graphoid Axioms]
Let $\Models = (G, \variables, P)$, with $\bm{W}, \bm{X}, \bm{Y}, \bm{Z} \subseteq \variables$. The (semi-)graphoid axioms are the following rules ($\Indep$ denotes $\Indep_P$ and $\Indep_G$)
\begin{enumerate}
	\item Symmetry: $\bm{X} \Indep \bm{Y} \mid \bm{Z} \Rightarrow \bm{Y} \Indep \bm{X} \mid \bm{Z}$.
	\item Decomposition: $\bm{X} \Indep \bm{Y} \cup \bm{W} \mid \bm{Z} \Rightarrow \bm{X} \Indep \bm{Y} \mid \bm{Z}$.
	\item Weak Union: $\bm{X} \Indep \bm{Y} \cup \bm{W} \mid \bm{Z} \Rightarrow \bm{X} \Indep \bm{Y} \mid \bm{W} \cup \bm{Z}$.
	\item Contraction: $(\bm{X} \Indep \bm{Y} \mid \bm{W} \cup \bm{Z}) \land (\bm{X} \Indep \bm{W} \mid \bm{Z}) \Rightarrow \bm{X} \Indep \bm{Y} \cup \bm{W} \mid \bm{Z}$.
\end{enumerate}
 For separations only on the graph, the graphoid axioms include two additional rules (only for $\Indep_G$).
\begin{enumerate}
	\setcounter{enumi}{4}
	\item Intersection: $(\bm{X} \Indep \bm{Y} \mid \bm{W} \cup \bm{Z}) \land (\bm{X} \Indep \bm{W} \mid \bm{Y} \cup \bm{Z}) \Rightarrow \bm{X} \Indep \bm{Y} \cup \bm{W} \mid \bm{Z}$, for any pairwise disjoint subsets $\bm{W}, \bm{X}, \bm{Y}, \bm{Z} \subseteq \variables$.
	\item Composition: $(\bm{X} \Indep \bm{Y} \mid \bm{Z}) \land (\bm{X} \Indep \bm{W} \mid \bm{Z}) \Rightarrow \bm{X} \Indep \bm{Y} \cup \bm{W} \mid \bm{Z}$.
\end{enumerate}
\end{definition}

As an illustration why certain rules only hold for graphs and not generally for probability distributions, consider rule (6) and Figure~\ref{fig:adjacency-failure}(a) again. From the distribution induced by the xor, we find that $Y \Indep_P X$ and $Y \Indep_P Z$ but we cannot conclude that $Y \Indep_P \{ X,Z \}$. If, however, in a graph $Y$ is $d$-separated from $X$ and from $Z$ then $Y$ is $d$-separated from the set $\{ X,Z \}$.

\leuftand*
\begin{proof}
Assume that w.l.o.g.~$X \nIndep_P \{ Y, Z \}$ is violated. By weak union, we get $X \Indep_P Y \mid Z$ which is equivalent to $Y \Indep_P X \mid Z$, using symmetry. We know that $Y \Indep_P Z$. By contraction, we get that $Y \Indep_P \{ X, Z \}$. Similarly, we conclude that $Z \Indep_P \{ X, Y \}$. Altogether, this implies that $X,Y,Z$ would be independent, which is a contradiction.

Each pair of joint dependence and marginal independence, e.g.~$X \nIndep_P \{ Y, Z \}$ and $X \Indep_P Z$, implies a conditional dependence, e.g.~$X \nIndep_P Y \mid Z$, by contraction.
\end{proof}

\leconnected*
\begin{proof}
Assume w.l.o.g. that $X$ is $d$-separated from $Y$ and $Z$ in $G$---i.e. $X \Indep_G Y$ and $X \Indep_G Z$. By applying the composition axiom, we get that $X \Indep_G \{ Y, Z \}$. If we apply the causal Markov condition, we get that $X \Indep_P \{ Y, Z \}$, which is a contradiction to our assumption.
\end{proof}
\fi

\thcollider*
\begin{proof}
There must be (at least) one node in $\{ X,Y,Z \}$ that is not an ancestor of any of the other nodes, say $Z \not \in \An(X)$ and $Z \not \in \An(Y)$, because of acyclicity. In other words, $X \not \in \De(Z)$ and $Y \not \in \De(Z)$. The local Markov property states that
\iflong
\else
$Z \Indep_G \Nd(Z) \mid \Pa(Z)$
\fi
and hence in particular
\[
Z \Indep_G \{ X,Y \} \mid \Pa(Z) \, .
\]
Further, if $| \Pa(Z) \cap \{X,Y \} | < 2$, we get a contradiction with the assumed conditional dependences. Hence $\{ X,Y \} \subseteq \Pa(Z)$ and $X \to Z \leftarrow Y$ is in $G$.
\end{proof}

\thsoundness*
\begin{proof}
First, we derive a general statement about the relations between $\bm{X}$ and $\bm{Z}$ without further specifying the role of $Y$. In particular, we show that there always exists a pair $(X,Z) \in \bm{X} \times \bm{Z}$ s.t. w.l.o.g.
\begin{equation}
\label{eq:pair:exists}
X \Indep_G Z \mid \Pa(X) \cup (\bm{X} \backslash \{ X \}) \cup (\bm{Z} \backslash \{ Z \}) \; ,
\end{equation}
where $\Pa(X) \subseteq \variables \backslash \bm{Z}$. 
Due to acyclicity, there has to exist a node in $\bm{X} \cup \bm{Z}$, say $X$, that is not an ancestor of any node in $(\bm{X} \cup \bm{Z}) \backslash \{ X \}$ and hence $(\bm{X} \cup \bm{Z}) \backslash \{ X \} \subseteq \Nd(X)$. 
By the local Markov condition, we get that $X \Indep_G (\bm{X} \cup \bm{Z}) \backslash \{ X \} \mid \Pa(X)$. 
Thus, by weak union,
$$X \Indep_G Z \mid \Pa(X) \cup (\bm{X} \backslash \{ X \}) \cup (\bm{Z} \backslash \{ Z \}) \; ,$$
for any $Z \in \bm{Z}$. Further, $\bm{Z} \cap \Pa(X) = \emptyset$, as by assumption no pair of nodes $(X,Z) \in \bm{X} \times \bm{Z}$ is adjacent in $G$.

Since $Y \aletwos \bm{X}$ and $Y \aletwos \bm{Z}$, we know that $Y$ is at least adjacent to one node in $\bm{X}$ and one node in $\bm{Z}$. Hence, $Y$ can take the following roles:
\begin{enumerate}[label=\alph*)]
	\item $Y$ is a descendent of each node in $\bm{X} \cup \bm{Z}$ (which corresponds to $\bm{X} \to Y \leftarrow \bm{Z}$),
	\item $Y$ is a non-descendent of each node in $\bm{X} \cup \bm{Z}$ and
	\item $Y$ is a descendent of at least one node in $\bm{X} \cup \bm{Z}$ and a non-descendent of at least one node in $\bm{X} \cup \bm{Z}$.
\end{enumerate}
The first statement corresponds to the graph structure implied by rule i) and any possible structure from the latter two is implied by the probabilities found in rule ii). To show these two implications hold, we do a proof by contraposition for each rule.

Hence, to show rule i), we need to prove that if the graph structure is not a collider---i.e.~$Y$ takes one of the roles described in b)~or c)---then there exists a pair $(X,Z) \in \bm{X}\times \bm{Z}$ and there exists a subset $\bm{S} \subseteq \variables \backslash \{ X,Z \}$ s.t.
\[
X \Indep_P Z \mid \bm{S} \cup \{ Y \} \cup (\bm{X} \backslash \{X \}) \cup (\bm{Z} \backslash \{Z \}) \; .
\]
First, consider all graphs in which $Y$ is a non-descendent of each node in $\bm{X} \cup \bm{Z}$ as described in b) We know from statement~(\refeq{eq:pair:exists}) that, w.l.o.g., there exists a pair $(X,Z) \in \bm{X}\times \bm{Z}$ for which $X \Indep_G Z \mid \Pa(X) \cup (\bm{X} \backslash \{ X \}) \cup (\bm{Z} \backslash \{ Z \})$. Since $Y \in \Nd(X)$, we will also find that $X \Indep_G Z \mid \Pa(X) \cup (\bm{X} \backslash \{ X \}) \cup (\bm{Z} \backslash \{ Z \}) \cup \{ Y \}$, where $\Pa(X)$ does not include $X$ or $Z$. Thus, by CMC we found the required independence. For the cases described in c), again assume that $X$ is not an ancestor of any node in $(\bm{X} \cup \bm{Z}) \backslash \{ X \}$. To conclude the same statement as previously, we show that $X$ has to be in $\De(Y)$ and thus $Y \in \Nd(X)$. We do this by deriving a contradiction: assume $X \in \Nd(Y)$. If $\bm{X}$ consists only of the single node $X$, then $X$ has to be adjacent to $Y$, $X \in \Pa(Y)$ and hence $X \to Y$ in $G$. Thus, $Y$ (and hence $X$) has to be an ancestor of at least one node in $\bm{Z}$, by assumption ($Y$ is a non-descendent of at least one node in $\bm{X} \cup \bm{Z}$), 
which is a contradiction. Similarly, if $\bm{X}$ contains a second node, $X'$, we know by assumption that $X' \in \Nd(X)$. We also know that the triple $\{ X,X',Y \}$ has to contain a collider. $X$ cannot be the collider, since $X \not \in \De(Y)$ and also $X'$ cannot be the collider since $X \not \in \An(X')$. Hence, $Y$ has to be the collider on the path $\langle X, Y, X' \rangle$. As above, at least one node $Z \in \bm{Z}$ has to be a descendent of $Y$, by assumption
and thus, $X \in \An(Z)$, which is a contradiction.

Last, we prove that the implication in rule ii) holds. Thus, by contraposition, we need to show that if $\bm{X} \to Y \leftarrow \bm{Z}$, then there exists a pair $X,Z \in \bm{X}\times \bm{Z}$ s.t. $X$ is conditionally independent of $Z$ given a subset of $\variables \backslash \{ X,Z \}$ that contains $(\bm{X} \backslash \{X \}) \cup (\bm{Z} \backslash \{Z \})$ but does not contain $Y$. From statement~(\refeq{eq:pair:exists}) there exists a pair $(X,Z) \in \bm{X} \times \bm{Z}$ that is $d$-separated given $\Pa(X) \cup (\bm{X} \backslash \{ X \}) \cup (\bm{Z} \backslash \{ Z \})$. Since $Y$ cannot be in $\Pa(X)$ due to acyclicity, we showed that there exists such a pair of nodes $X,Z$ that can be rendered conditionally independent by a subset of $\variables \backslash \{ X,Z \}$ that contains $(\bm{X} \backslash \{X \}) \cup (\bm{Z} \backslash \{Z \})$ but does not contain $Y$
\iflong
\else
(after applying CMC).
\fi
\end{proof}

\cordetection*
\begin{proof}
Since we know that $Y \aletwos \bm{X}$ and $Y \aletwos \bm{Z}$, we can conclude that, as in the proof of Theorem~\ref{th:soundness}, $Y$ can take three different roles w.r.t. $\bm{X}$ and $\bm{Z}$, where role a) corresponds to condition i) in $2$-orientation faithfulness and rule i) in the orientation rule and roles b)~and c)~correspond to condition ii) and rule ii). 

Now assume that condition i) in $2$-orientation faithfulness fails, that is, the true graph can be described by role a), but there exists a pair $X \in \bm{X}$ and $Z \in \bm{Z}$, for which $X$ is independent of $Z$ given a subset of $\variables \backslash \{ X,Z \}$ that contains ${Y} \cup (\bm{X} \backslash \{X \}) \cup (\bm{Z} \backslash \{Z \})$. If this is the case, we cannot apply rule i) of our orientation rule. In addition, we showed in Theorem~\ref{th:soundness} that for a graph as described by a) rule ii) can never apply. Thus, we can detect this failure of condition i) in $2$-orientation faithfulness by noticing that neither rule i) nor ii) of our orientation rule applies.

Next, assume condition ii) in $2$-orientation fails. This means that we cannot apply rule ii) of the orientation rule. Again, we showed that for such graphs $Y$ takes either role b)~or c), in which case orientation rule i) can never apply. Hence, we can detect if condition ii) in $2$-orientation faithfulness fails, since none of the conditions in the orientation rule is met.
\end{proof}

\thcorrectness*
\begin{proof}
We follow the original correctness proof under the faithfulness assumption~\citep{margaritis:00:gs}, that consists of two main steps. First, we need to show that $\MB(T) \subseteq \bm{S}$ after the grow phase and second, we need to ensure that all nodes in $\MB(T)$ stay in $\bm{S}$ during the shrink phase, while nodes not in $\MB(T)$ will be removed from $\bm{S}$ in the shrink phase.

Grow phase: By assumption ($2$-adjacency faithfulness), for each node $X \in \PC(T)$, $T$ is either $1$-associated to $X$, or there exists a set $\bm{X}$ that includes $X$ such that $T \atwos \bm{X}$. If $T$ is $1$-associated to a node $X$, then $T \nIndep_P X \mid \bm{S}$, if $X \not \in \bm{S}$, hence we will add those nodes. If $T$ is strictly $2$-associated to a set $\{ X,Z \}$ then $T \nIndep_P X \mid \bm{S} \cup \{ Z \}$ for all $\bm{S} \subseteq \variables \backslash \{ X,T,Z \}$. Thus, we also add $X$ to $\bm{S}$, if $X \not \in \bm{S}$ and afterwards also find that $T \nIndep_P Z \mid \bm{S}$, if $Z \not \in \bm{S}$, since $X \in \bm{S}$. Hence, all nodes in $\PC(T)$ will be added during the grow phase. Next, we need to consider the spouses of $T$ that do not overlap with $\PC(T)$, hence might not have been added yet.\!\footnote{There could be nodes that are spouses of $T$ and in $\PC(T)$ at the same time e.g. if $T$ has two children $X$ and $Z$, where $Z$ is also a parent of $X$.} Since we know that eventually $\bm{S}$ will contain all children of $T$, we will afterwards also add the corresponding spouses. In particular, we need to consider two classes of spouses $S$: 1) Spouses that through a child node $C$ are strictly $2$-associated to $T$ ($T \atwos \{C,S \}$). Those will be added due to the strict $2$-association as explained above. 2) Spouses that are not involved in such a strict $2$-association. For the latter, we find a conditional dependence between $T$ and $S$ by conditioning on the corresponding child node $C$ (by Assumption~\ref{as:one}), which will be in $\bm{S}$. A special case occurs if a child node $C$ is strictly $2$-associated to two spouses $S_1$ and $S_2$. Due to Assumption~\ref{as:one}, $T$ is dependent on $S_1$ if we condition on $C$ and $S_2$, vice versa $T$ is dependent on $S_2$ if we condition on $C$ and $S_1$. Similarly to how we add strict $2$-associations above, we will also first add one of the two and then the second one.
Thus, after the grow phase, $\bm{S}$ will contain all elements of $\MB(T)$.

Shrink phase: Since it is possible that after the grow phase $\bm{S}$ is a superset of $\MB(T)$, we need to ensure that in the shrink phase all $W \not \in \MB(T)$ will be deleted from $\bm{S}$ and all $X \in \MB(T)$ will stay in $\bm{S}$.

First, we show that no node $X \in \MB(T)$ will be removed from $\bm{S}$.
Assume $X$ is the first element in $\MB(T)$ that we attempt to remove from $\bm{S}$. If $X \in \PC(T)$, by definition of $2$-adjacency faithfulness $T$ is either $1$-associated to $X$ and hence, $X$ will not be removed, or $T$ is strictly $2$-associated to a set $\bm{X} \subseteq \MB(T)$ that contains $X$. W.l.o.g. let $\bm{X} = \{ X,Z \}$, then $T \nIndep_P X \mid \bm{S} \backslash \{ X \}$, since $\bm{S}$ contains $Z$, and hence, $X$ will not be removed from $\bm{S}$.
If $X$ is a spouse of $T$, there again exist two cases. Either $T$ is strictly $2$-associated to a set that contains $X$, in which case, $X$ will not be removed from $\bm{S}$ as explained above, or $T$ is not strictly $2$-associated to a set that contains $X$. In the latter case, by Assumption~\ref{as:one}, $X$ is dependent on $T$ conditioned on a subset of $\MB(T) \backslash \{ X \}$ and thus $X \nIndep_P T \mid \bm{S} \backslash \{ X \}$. In particular, this subset consists of the common child $C$ and in the special case that $C$ is strictly $2$-associated to $X$ and a second spouse $S$, it also contains that second spouse $S$. Either way, those conditioning sets are contained in $\bm{S}$. Hence, $X$ will not be removed from $\bm{S}$. In the following iterations, $\bm{S}$ will still contain $\MB(T)$ and hence, we will also not remove a true element of $\MB(T)$.

Last, assume $W \not \in \MB(T)$, but $W \in \bm{S}$ after the grow phase. Further, we can write $\bm{S} \backslash \{ W \}$ as $\MB(T) \cup \bm{Q}$, where $\bm{Q}$ contains all elements from $\bm{S} \backslash \{ W \}$ that are not in $\MB(T)$. Then, 
$T \Indep_G \{ W \} \cup \bm{Q} \mid \MB(T)$ and thus by weak union, $T \Indep_G W \mid \MB(T) \cup \bm{Q}$, which implies $T \Indep_P W \mid \bm{S} \backslash \{ W \}$ (by CMC). Hence, we delete each node in $\bm{S}$ that is not in $\MB(T)$ in the shrink phase.
\end{proof}

\end{document}